\documentclass[10pt]{article}
\usepackage[top=1in, bottom=1in, left=0.9in, right=0.9in]{geometry}

\usepackage{zhiyuz}

\title{\textbf{Operationalizing Stein's Method for Online Linear Optimization: CLT-Based Optimal Tradeoffs}}
\author{
Zhiyu Zhang\\
Carnegie Mellon University\\
\texttt{zhiyuzresearch@gmail.com}\\
\and
Aaditya Ramdas\\
Carnegie Mellon University\\
\texttt{aramdas@cmu.edu}
}
\date{\vspace{-5ex}}

\begin{document}
\maketitle

\begin{abstract}
Adversarial online linear optimization (OLO) is essentially about making performance tradeoffs with respect to the unknown difficulty of the adversary. In the setting of one-dimensional fixed-time OLO on a bounded domain, it has been observed since Cover \citep{cover1966behavior} that achievable tradeoffs are governed by probabilistic inequalities, and these descriptive results can be converted into algorithms via dynamic programming, which, however, is not computationally efficient. We address this limitation by showing that Stein's method, a classical framework underlying the proofs of probabilistic limit theorems, can be operationalized as computationally efficient OLO algorithms. The associated regret and total loss upper bounds are ``additively sharp'', meaning that they surpass the conventional big-O optimality and match normal-approximation-based lower bounds by additive lower order terms. Our construction is inspired by the remarkably clean proof of a Wasserstein martingale central limit theorem (CLT) due to R{\"o}llin \citep{rollin2018quantitative}.

Several concrete benefits can be obtained from this general technique. First, with the same computational complexity, the proposed algorithm improves upon the total loss upper bounds of online gradient descent (OGD) and multiplicative weight update (MWU). Second, our algorithm can realize a continuum of optimal two-point tradeoffs between the total loss and the maximum regret over comparators, improving upon prior works in parameter-free online learning. Third, by allowing the adversary to randomize on an unbounded support, we achieve sharp in-expectation performance guarantees for OLO with noisy feedback. 
\end{abstract}

\section{Introduction}\label{section:introduction}

We study a fundamental protocol in adversarial online learning: one-dimensional fixed-time \emph{online linear optimization} (OLO) on a bounded domain. With a time horizon $T\in\N_+$, it is a repeated game between an algorithm we design and an adversary. The algorithm is defined by a sequence of functions $(a_t)_{t\in[1:T]}$, where each $a_t:\R^{t-1}\rightarrow[-1,1]$ can depend on $T$. The adversary is defined by another sequence of functions $(l_t)_{t\in[1:T]}$, where each $l_t:\R^t\rightarrow\R$ can depend on $T$ and the algorithm $(a_t)_{t\in[1:T]}$. In each (the $t$-th; $t\in[1:T]$) round,
\begin{itemize}
\item The algorithm picks the decision $x_t=a_t(g_{1:t-1})\in[-1,1]$. 
\item The adversary picks the loss gradient $g_t=l_t(x_{1:t})\in\R$. 
\item The algorithm observes $g_t$ and then suffers the linear loss $g_tx_t$. 
\end{itemize}
At the end of the game, the performance of the algorithm is measured by its \emph{total loss}, $\loss_T\defeq \sum_{t=1}^Tg_tx_t$, lower is better. A real world example is sequential betting: $x_t$ is the amount of money a gambler bets on the negativity of $g_t$, and with the instantaneous payoff being $-g_tx_t$, the total amount of money they make is $-\loss_T$. 

For intuition, let us briefly assume $g_t\in[-1,1]$ such that the worst case adversary would pick $g_t=\sign(x_t)$. It is clear that optimizing only for this worst case would trivialize the problem, since the minimax-optimal algorithm is forced to always pick $x_t=0$. The convention instead is to simultaneously consider adversaries with varied difficulty: although any nontrivial algorithm should suffer positive $\loss_T$ in the worst case, one could indeed expect negative $\loss_T$ if the adversary ends up being easy. Therefore every algorithm should realize a tradeoff, just like in betting where higher profits necessarily comes with higher risks. The objective is thus achieving suitable notions of ``admissibility'', in the sense that no other algorithm can always guarantee lower $\loss_T$ regardless of the adversary. 

Rigorously studying this objective motivates a common surrogate performance metric called the \emph{regret}: for all $u\in[-1,1]$ called a \emph{comparator}, the regret of the algorithm with respect to $u$ is defined as $\reg_T(u)\defeq\sum_{t=1}^Tg_t(x_t-u)$. If an algorithm guarantees $\reg_T(u)\leq \psi_T(u)$ for some proper, closed and convex function $\psi_T:[-1,1]\rightarrow(-\infty,\infty]$,\footnote{Defined as $\infty$ outside $[-1,1]$ following the convention.} then with the convex and $1$-Lipschitz function $\psi^*_T:\R\rightarrow\R$ being the convex conjugate of $\psi_T$, $\loss_T$ can be upper-bounded by the oracle inequality
\begin{equation}\label{eq:loss_bound}
\loss_T\leq \inf_{u\in[-1,1]} \spar{\rpar{\sum_{t=1}^Tg_t}u+\psi_T(u)}=-\sup_{u\in\R} \spar{\rpar{-\sum_{t=1}^Tg_t}u-\psi_T(u)}=-\psi_T^*\rpar{-\sum_{t=1}^Tg_t}.
\end{equation}
Since $(\psi^*_T)^*=\psi_T$, the total loss bound conversely implies the regret bound, and this two-way relation is called the \emph{loss-regret duality} \citep{mcmahan2014unconstrained} (Lemma~\ref{lemma:duality}). The idea is that upper-bounding the surrogate $\reg_T(u)$ is equivalent to upper-bounding $\loss_T$ itself in an adversary-dependent manner, where the cumulative bias $\sum_{t=1}^Tg_t$ can be used to quantify the difficulty of the adversary.\footnote{Intuitively, a larger $|\sum_{t=1}^Tg_t|$ means the adversary is more biased thus easier for the algorithm to exploit. Therefore in Eq.\eqref{eq:loss_bound}, the typical $\psi^*_T$ function of interest would be an even function minimized at $0$, meaning the loss bound is maximized at $\sum_{t=1}^Tg_t=0$.} We will adopt this ``loss-centric'' viewpoint to characterize performance tradeoffs in OLO, whose relation to the conventional view of regret minimization is demonstrated through the following example. 

\begin{example}[Uniform regret]\label{example:worst}
Consider the following definition which we call the \emph{uniform regret},
\begin{equation*}
\reg^\unif_T\defeq \sup_{u\in[-1,1]}\reg_T(u)=\max\left\{\reg_T(-1),\reg_T(1)\right\}.
\end{equation*}
Any regret bound $\reg_T(u)\leq \psi_T(u)$ can be converted into a uniform regret bound $\reg^\unif_T\leq \sup_{u\in[-1,1]}\psi_T(u)$, whose right hand side, if finite, can be regarded as a function $\psi^\unif_T$ invariant to its argument. Then, an upper bound on $\loss_T$ is induced by the convex conjugate of $\psi^\unif_T$, $(\psi^\unif_T)^*(s)=\abs{s}-\sup_{u\in[-1,1]}\psi_T(u)$. Despite being weaker than $\psi^*_T$, i.e., $(\psi^\unif_T)^*(s)\leq \psi^*_T(s)$ for all $s\in\R$, this is a convenient relaxation adopted by most prior works, especially due to its connection to uniform convergence in statistical learning. See, e.g., \citep{rakhlin2014statistical}.
\end{example}

This work presents computationally efficient algorithms achieving general and sharp \emph{upper bound functions} on $\reg_T(u)$ and $\loss_T$. En route, we introduce the techniques of \emph{Stein's method} \citep{stein1972bound,stein1986approximate} to adversarial online learning and demonstrate intriguing algorithmic connections between sharp rates and \emph{central limit theorems} (CLTs). 

\subsection{Motivation}\label{subsection:motivation}

To concretely motivate our results, we start from a central problem of the field: characterizing the conditions on $\psi_T^*$ for the achievability of Eq.\eqref{eq:loss_bound} against a large class of adversaries. In the special case of Boolean adversaries, the following landmark result is due to Cover \citep[Statement V]{cover1966behavior}. 

\begin{theorem}[Cover's characterization, adapted]\label{theorem:cover}
Assume that $g_t\in\{-1,1\}$ for all $t\in[1:T]$, and define $\rs(n)$ as the distribution of the sum of $n$ independent Rademacher random variables. Then, for all convex and $1$-Lipschitz function $\psi^*_T:\R\rightarrow(-\infty,\infty]$, there exists an algorithm achieving the total loss bound Eq.\eqref{eq:loss_bound} if and only if
\begin{equation}\label{eq:achievability_cover}
\E_{X\sim \rs(T)}[\psi_T^*(X)]\leq 0.
\end{equation}
In particular, the corresponding algorithm outputs the expectation of a discrete derivative, 
\begin{equation}\label{eq:cover}
x_t=-\half\E_{X\sim\rs(T-t)}\spar{\psi^*_T\rpar{\sum_{i=1}^{t-1}g_i+X+1}-\psi^*_T\rpar{\sum_{i=1}^{t-1}g_i+X-1}},\quad\forall t\in[1:T].
\end{equation}
\end{theorem}

The necessity of Eq.\eqref{eq:achievability_cover} is straightforward, since any gambler cannot expect to make profit when the nature is uniformly random. Truly remarkable is the sufficiency of Eq.\eqref{eq:achievability_cover}, proved constructively by Eq.\eqref{eq:cover}. The construction reveals a fundamental dynamic programming principle which has inspired numerous subsequent works connecting repeated games, stochastic analysis and partial differential equations (PDEs) \citep{cesa2006prediction,rakhlin2014statistical,karlin2017game,drenska2020prediction,harvey2023optimal}. Based on that Cover proved the following in the same paper: still assuming $g_t\in\{-1,1\}$ for all $t$, an instantiation of Eq.\eqref{eq:cover} achieves the uniform regret bound $\reg^\unif_T\leq \sqrt{\frac{2}{\pi}T}+1$, where the leading constant $\sqrt{\frac{2}{\pi}}$ cannot be improved further. This is better than the best known $\reg^\unif_T\leq\sqrt{T}$ bound of \emph{online gradient descent} (OGD) \citep{zinkevich2003online}, and when adapted to the problem of \emph{learning from expert advice} (LEA) with two experts, it is also better than the uniform regret bound of \emph{multiplicative weight update} (MWU) \citep{littlestone1994weighted,arora2012multiplicative}. 

However, a crucial limitation of Cover's algorithm is the computation. While OGD and MWU both take $O(1)$ time per round, exactly evaluating Eq.\eqref{eq:cover} requires sweeping over the size-($T-t+1$) support of $\rs(T-t)$ therefore is $O(T)$ time. Such a gap becomes more evident if $g_t\in\{-\eps_t,\eps_t\}$ for some distinct $\eps_1,\ldots,\eps_T\in\R_{>0}$: the distribution underlying the accordingly generalized Eq.\eqref{eq:cover} has support size $2^{T-t}$, making the computation of $x_t$ exponential-time. Furthermore, Cover's algorithm requires the prior knowledge of $\abs{g_t}$ for all $t$, whereas OGD and MWU require only upper bounds on these $\abs{g_t}$, at most.\footnote{Adaptive versions of OGD and MWU such as \textsc{AdaGrad} \citep{duchi2011adaptive} do not require a priori upper bounds on $\abs{g_t}$. The price to pay is a larger leading constant in the regret bound. These are orthogonal to the focus of our work.} The intuition is that some kind of continuous approximation is necessary, but a general, sharp and analytically scalable solution remains challenging. 
\begin{itemize}
\item A natural attempt motivated by CLTs would be replacing the distribution $\rs(T-t)$ in Eq.\eqref{eq:cover} by the normal $\calN(0,T-t)$, such that computing the Gaussian integral $\E_{X\sim\calN(0,T-t)}[\psi^*_T(s+X)]$ can be assumed to take $O(1)$ time. The Lipschitzness of $\psi^*_T$ gives the normal approximation error $|\E_{X\sim\calN(0,T-t)}[\psi^*_T(s+X)]-\E_{X\sim\rs(T-t)}[\psi^*_T(s+X)]|=O(1),\forall s\in\R$ \citep[Theorem~2]{zhu2014two}, but simply combining it with Eq.\eqref{eq:cover} introduces $O(T)$ cumulative error in the associated regret and total loss bounds. 
\item Existing works \citep{kobzar2020a_new,greenstreet2022efficient} refined this idea for the special case of Example~\ref{example:worst}. Building on the relation between $(\psi^\unif_T)^*$ and the absolute value function, the proposed algorithms are based on querying the \emph{error function} $\erf:\R\rightarrow\R$ which is computationally efficient. Via specialized Taylor-approximation-based proofs, they also match the leading $\sqrt{\frac{2}{\pi}T}$ term in Cover's uniform regret bound against the larger class of $g_t\in[-1,1]$ adversaries.
\end{itemize}

\begin{figure}[t]
\includegraphics[width=\textwidth]{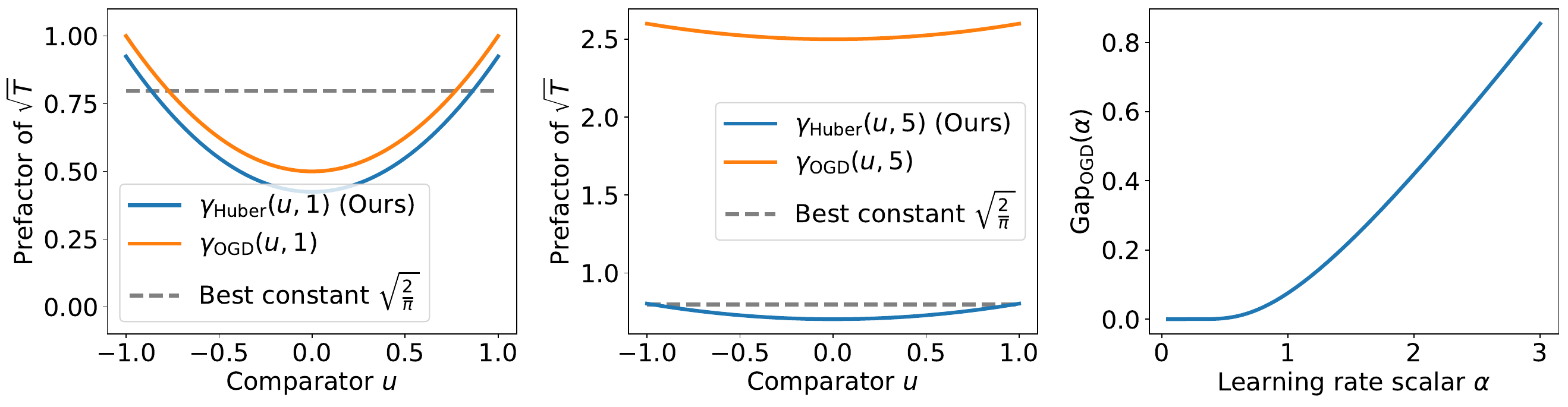}
\caption{Comparison of the prefactors of $\sqrt{T}$ in the regret bounds of our algorithm (denoted by $\gamma_{\mathrm{Huber}}(u,\alpha)$), OGD (denoted by $\gamma_{\mathrm{OGD}}(u,\alpha)$), and Cover's algorithm (the optimal $u$-independent prefactor $\sqrt{\frac{2}{\pi}}$; computationally efficient versions are provided by \citep{kobzar2020a_new,greenstreet2022efficient}). $u$ is the comparator and $\alpha$ is the scaling factor of the learning rate; see Section~\ref{subsection:huber} for definitions. Left: with $\alpha=1$, $\gamma_{\mathrm{Huber}}(u,\alpha)$ and $\gamma_{\mathrm{OGD}}(u,\alpha)$ are compared as functions of $u$, lower is better. Our regret bound dominates that of OGD for all $u\in[-1,1]$, while the optimal $u$-independent bound does not. Middle: with $\alpha=5$, the improvement over OGD becomes more significant. Right: the margin of improvement $\mathrm{Gap}_{\mathrm{OGD}}(\alpha)\defeq\gamma_{\mathrm{OGD}}(u,\alpha)-\gamma_{\mathrm{Huber}}(u,\alpha)$ as a function of $\alpha$. 
}
\label{figure:ogd}
\end{figure}

An immediate follow-up question is whether other \emph{additively sharp} guarantees beyond Example~\ref{example:worst} can be efficiently realized. Here ``additively sharp'' means that the obtained upper bound on $\loss_T$ should be away from a probabilistic achievability boundary (e.g., Eq.\eqref{eq:achievability_cover}) by additive, rather than multiplicative, lower order factors. As motivated previously, this corresponds to the efficient and near-optimal realization of general performance tradeoffs, which is of central importance in OLO. Moreover, both OGD and MWU are by default equipped with certain comparator-dependent regret bound $\reg_T(u)\leq \psi_T(u)$, therefore instead of beating their respective relaxation $\reg^\unif_T\leq\sup_{u\in[-1,1]}\psi_T(u)$ which is certainly a loose baseline, it is more sensible to compare the regret bound to such $\psi_T$ functions and aim for pointwise dominance (see Figure~\ref{figure:ogd}, Left). By the loss-regret duality, this will ultimately ensure a dominance on the corresponding $\loss_T$ upper bounds. 

In this work we present a solution to this problem, among others. The connection to normal approximation motivates the use of Stein's method, a classical analytical framework underlying various nonasymptotic, quantitative versions of CLTs. 

\subsection{Summary of Results}\label{subsection:results}

\paragraph{Quantitative} This subsection assumes $g_t\in[-1,1]$ for all $t$ unless stated otherwise. Our main result is a normal-approximation-based characterization of performance tradeoffs in OLO. 

\begin{theorem}[Theorem~\ref{theorem:main} and \ref{theorem:loss_lower}, informal]\label{theorem:informal}
Let $\psi^*_T:\R\rightarrow\R$ be an arbitrary convex and $1$-Lipschitz function satisfying $\E_{X\sim \calN(0,T)}[\psi_T^*(X)]=0$. There exists an $O(1)$-time-per-round algorithm (Algorithm~\ref{algorithm:main}) guaranteeing
\begin{equation*}
\loss_T\leq -\psi^*_T\rpar{-\sum_{t=1}^Tg_t}+O(\log T),
\end{equation*}
while it is impossible for any algorithm to guarantee $\loss_T\leq -\psi^*_T(-\sum_{t=1}^Tg_t)-\Omega(1)$.
\end{theorem}

Compared to Cover's characterization (Theorem~\ref{theorem:cover}), the above result overcomes the computational limitation of dynamic programming at the price of an $O(\log T)$ gap between the upper and lower bounds. Since the $\psi^*_T$ function typically grows with rate $\Theta(\sqrt{T})$, such a gap is on a lower order, meaning that the upper bound is ``additively sharp''. Based on this master theorem, the following corollaries and extensions are obtained. 
\begin{itemize}
\item \emph{Dominating OGD and MWU.}\quad For any hyperparameter $\alpha\in\R_{>0}$, a special case of our algorithm guarantees $\reg_T(u)\leq \gamma_{\mathrm{Huber}}(u,\alpha)\sqrt{T}+O(\log T)$, where the prefactor of $\sqrt{T}$ is given by
\begin{equation*}
\gamma_{\mathrm{Huber}}(u,\alpha)=\frac{u^2}{2\alpha}+\rpar{\alpha+\frac{1}{\alpha}} \Phi\rpar{\frac{1}{\alpha}}+\phi\rpar{\frac{1}{\alpha}}-\half\alpha-\frac{1}{\alpha}.
\end{equation*}
This is compared to the OGD baseline with constant learning rate $\alpha/\sqrt{T}$, whose best known regret bound has the classical $\sqrt{T}$-prefactor $\gamma_{\mathrm{OGD}}(u,\alpha)=\half(\alpha^{-1}u^2+\alpha)$. We show that $\gamma_{\mathrm{Huber}}(u,\alpha)<\gamma_{\mathrm{OGD}}(u,\alpha)$ for all $u\in[-1,1]$ and $\alpha\in\R_{>0}$, and furthermore, $\lim_{\alpha\rightarrow\infty}\gamma_{\mathrm{Huber}}(u,\alpha)=\sqrt{\frac{2}{\pi}}$ while $\lim_{\alpha\rightarrow\infty}\gamma_{\mathrm{OGD}}(u,\alpha)=\infty$. Combined with Eq.\eqref{eq:loss_bound}, it means the total loss upper bound of our algorithm dominates that of OGD, and their gap can be arbitrarily large. In contrast, improving the uniform regret bound of OGD as in \citep{kobzar2020a_new,greenstreet2022efficient} does not warrant such a dominance. Details are provided in Section~\ref{section:dominate_ogd}, and see Figure~\ref{figure:ogd} for visualizations. 

The same argument can be applied to improve the MWU baseline. Notably, a quantitative separation is obtained by comparing the regret upper bound of our algorithm to an existing regret lower bound of MWU \citep{gravin2017tight}. See Appendix~\ref{section:dominate_mwu}. 

\item \emph{Loss versus uniform regret.}\quad Within performance tradeoffs in OLO, a notable special case first studied by Even-Dar et al. \citep{even2008regret} is the optimal ``two-point'' tradeoff between $\loss_T$ and $\reg^\unif_T$, i.e., the regret with respect to the average comparator $u=0$ versus the best comparator $u\in\{-1,1\}$. With $\Phi:\R\rightarrow(0,1)$ being the standard normal CDF, we show that for any constant $\eps\in(0,\sqrt{\frac{\pi}{2}}]$, an instantiation of our algorithm guarantees both $\loss_T\leq \eps\sqrt{T}+O(\log T)$ and $\reg^\unif_T\leq \gamma(\eps)\sqrt{T}+O(\log T)$, where $\gamma(\eps)\in[\sqrt{\frac{\pi}{2}},\infty)$ is the unique solution of the equation $\int_{-\infty}^{\eps-\gamma(\eps)}\Phi(x)\diff x=\frac{\eps}{2}$ and cannot be improved further. This improves upon the state-of-the-art approach combining an unconstrained ``parameter-free'' OLO algorithm \citep{zhang2022pde} with a constrained-to-unconstrained reduction \citep{cutkosky2018black}. See Appendix~\ref{section:soft_thresholded}.

\item \emph{Noisy feedback.}\quad Existing additively sharp guarantees as well as our results above are based on the assumption of deterministically bounded adversaries. To build an even stronger parallel relation with CLTs, we further allow the adversary to randomize on an unbounded domain, thereby giving sharp in-expectation performance guarantees for OLO with \emph{noisy feedback}. Despite the differences in their settings, the obtained result can be viewed as the algorithmic analogue of a nonasymptotic Wasserstein martingale CLT \citep{rollin2018quantitative}, which is of notable conceptual value. See Appendix~\ref{section:stochastic}.
\end{itemize}

\paragraph{Methodological} Our results are enabled by introducing Stein's method \citep{stein1972bound} to OLO. For context, Stein's method refers to a collection of differential-operator-based techniques characterizing the distance of a distribution to a reference, such as the normal. Within probability theory, it is a canonical framework for proving quantitative limit theorems \citep{chen2010normal,nourdin2012normal}, often being simpler than the Taylor-approximation-based alternatives. Our key observation is that the strengths of Stein's method are highly congruent with the analytical difficulties in OLO, such that the proofs of certain probabilistic limit theorems, which are ``descriptive'' in nature, can be operationalized as ``prescriptive'' OLO algorithms with sharp rates. Specifically for this work, our construction parallels the remarkably clean proof of a Wasserstein martingale CLT due to R{\"o}llin \citep{rollin2018quantitative}, while we believe the applicability of this idea extends much further. See Section~\ref{section:analysis}. 

\subsection{Organization}

Section~\ref{section:algorithm} introduces the template of our algorithm, followed by the analysis in Section~\ref{section:analysis} which is the crux of this work. Section~\ref{section:dominate_ogd} presents an instantiation of the algorithm improving upon OGD. Section~\ref{section:conclusion} concludes this work and discusses future directions. Other quantitative results and discussions of related works are deferred to the appendix; see the beginning of the appendix for its organization. 

\paragraph{Notation} $Z$ represents the standard normal random variable with cumulative distribution function (CDF) $\Phi$ and probability density function (PDF) $\phi$. More generally, $\Phi_{\mu,\sigma}$ and $\phi_{\mu,\sigma}$ represent the CDF and PDF of $\calN(\mu,\sigma^2)$. For all Lebesgue-measurable $f:\R\rightarrow\R$, let $\norm{f}_\infty\defeq\mathrm{ess}\sup_{x\in\R}\abs{f(x)}$ be its essential supremum norm (with respect to the Lebesgue measure). Let $\Pi_{[a,b]}(x)=\argmin_{u\in[a,b]}\abs{u-x}$ be the projection function to the interval $[a,b]$. For integers $a\leq b$, $[a:b]$ and $x_{a:b}$ respectively denote the tuples $[a,\ldots,b]$ and $[x_a,\ldots,x_b]$. Let $\partial_1$ and $\partial_2$ be the partial derivatives of a function with respect to the first and the second argument; similarly, $\partial_{11},\partial_{12},\partial_{22}$ are the second order partial derivatives. 

\section{Algorithm}\label{section:algorithm}

\paragraph{Stein equation} We begin by introducing the key element of Stein's method, \emph{Stein equation} and its solution. Appendix~\ref{section:background_stein} provides the minimal technical background to keep the present work self-contained, while interested readers are referred to several canonical resources for comprehensive treatments \citep{chen2010normal,ross2011fundamentals,nourdin2012normal}. 

\begin{definition}[Solution of Stein equation]\label{definition:stein_equation}
Let $h:\R\rightarrow\R$ be a convex and $1$-Lipschitz function. For all $\mu\in\R$ and $\sigma\in\R_{>0}$, consider the linear ordinary differential equation
\begin{equation}\label{eq:stein}
\sigma^2f'(x)-(x-\mu)f(x)=h(x)-\E_{Z\sim \calN(0,1)}[h(\mu+\sigma Z)],\quad \forall x\in\R,
\end{equation}
called the Stein equation associated with the distribution $\calN(\mu,\sigma^2)$ and the target function $h$. It has a unique bounded solution, and we denote it by $f_{\mu,\sigma,h}:\R\rightarrow\R$. 
\end{definition}

The function $f_{\mu,\sigma,h}$ can be expressed as integrals using $h$, but due to the modularity of Stein's method, our analysis is performed at an abstract level without using $h$-specific structures. This is a major benefit of Stein's method over more granular alternatives. 

We also note the following properties of $f_{\mu,\sigma,h}$. Due to the Lipschitzness of $h$, the derivative $f'_{\mu,\sigma,h}$ is Lipschitz, therefore the second derivative $f''_{\mu,\sigma,h}$ exists almost everywhere \citep[Proposition~3.5.1]{nourdin2012normal}. Likewise, $h'$ exists almost everywhere and $\norm{h'}_\infty$ equals the minimum Lipschitz constant of $h$. A concept called \emph{Stein factors} \citep{gaunt2025stein} refers to upper bounds on $||f_{\mu,\sigma,h}||_\infty$, $||f'_{\mu,\sigma,h}||_\infty$ and $||f''_{\mu,\sigma,h}||_\infty$ expressed using $\norm{h'}_\infty$; see Lemma~\ref{lemma:stein_magnitude}. We will immediately use $\norm{f_{\mu,\sigma,h}}_\infty\leq\norm{h'}_\infty\leq 1$.

\begin{algorithm}[ht]
\caption{1D fixed-time bounded-domain OLO via Stein's method.\label{algorithm:main}}
\begin{algorithmic}[1]
\REQUIRE A convex and $1$-Lipschitz function $h:\R\rightarrow\R$. 
\STATE Define $s_0\defeq 0$, and pick an arbitrary $\rho_0\in\R_{>0}$. 
\FOR{$t=1,\ldots,T$}
\STATE Pick an arbitrary $\rho_t\in[0,\rho_{t-1}]$, and we further require $\rho_t>0$ if $t<T$. Define $c_t\defeq\rho^2_{t-1}-\rho^2_t\in\R_{\geq 0}$. 
\STATE Based on the function $f_{\mu,\sigma,h}$ from Definition~\ref{definition:stein_equation}, output
\begin{equation}\label{eq:algorithm}
x_{t}=\E_{Z\sim\calN(0,1)}\spar{f_{s_{t-1},\rho_{t-1},h}(s_{t-1}+\rho_tZ)}.
\end{equation}
This is a valid output on the domain $[-1,1]$ since $||f_{s_{t-1},\rho_{t-1},h}||_\infty\leq 1$. 

\STATE Observe the loss gradient $g_t\in\R$ and define $s_{t}\defeq s_{t-1}+g_t\in\R$.
\ENDFOR
\end{algorithmic}
\end{algorithm}

\paragraph{Algorithm} With the above, we present Algorithm~\ref{algorithm:main}. A number of remarks are in order. 
\begin{itemize}
\item The function $h$ serves as a proxy of $\psi^*_T$ from the targeted Eq.\eqref{eq:loss_bound}, specifying the desirable tradeoff on the total loss. The algorithm is invariant to vertical shifts on $h$, i.e., $h(x)\leftarrow h(x)+\mathrm{constant}$ for all $x\in\R$. 
\item The parameter $\rho_t$ serves as a guess on the adversary's ``standard deviation to go'', $\sqrt{\sum_{i=t+1}^Tg_i^2}$, and it is allowed to depend on $g_{1:t-1}$. Requiring $\rho_t>0$ for all $t<T$ is to ensure the well-posedness of Eq.\eqref{eq:algorithm}. Analogously, the induced $c_t=\rho^2_{t-1}-\rho^2_t$ is a guess on the adversary's instantaneous variance, $g_t^2$. 

For intuition, if it is known that $\abs{g_t}\leq 1$ for all $t$, then one could simply set $\rho_t=\sqrt{T-t}$ thus $c_t=1$. We remark that in its full generality, hard constraints on the adversary are not required, as the associated performance guarantee (Theorem~\ref{theorem:main}) will hold against arbitrary unconstrained adversaries, taking the misspecification error of $\rho_{0:T}$ into account.
\end{itemize}

While the output Eq.\eqref{eq:algorithm} itself is sufficient for our analysis, its reliance on $f_{\mu,\sigma,h}$ may obfuscate the intuition. The following lemma spells out the explicit dependence of $x_t$ on the function $h$, and based on that, we discuss the intuition by analyzing its small-$c_t$ approximation. 

\begin{lemma}[Lemma~\ref{lemma:equivalent}, simplified]\label{lemma:equivalent_short}
Let $\mathrm{Exp}(1)$ be the exponential distribution with parameter $1$. The output Eq.\eqref{eq:algorithm} of Algorithm~\ref{algorithm:main} is equivalent to
\begin{equation*}
x_{t}=-\E_{Z\sim\calN(0,1);\tau\sim\mathrm{Exp}(1)}\spar{h'\rpar{s_{t-1}+\sqrt{\rho^2_{t-1}-e^{-2\tau} c_t}Z}}.
\end{equation*}
\end{lemma}

Suppose for some $t\in[1:T]$ we have $\rho_{t-1}=\sqrt{T-t+1}$ and $c_t\approx 0$. By Lemma~\ref{lemma:equivalent_short}, $x_t\approx -\E_{Z}[h'(s_{t-1}+\sqrt{T-t+1}Z)]$, which we define as the output $\tilde x_t$ of a simpler, $c_t$-independent analogue of Algorithm~\ref{algorithm:main}. This can be justified from the following two perspectives. 

\paragraph{Universality and regularization} The fundamental difficulty in OLO originates from the uncertainty in the adversary. Consider the $t$-th round where $g_{1:t-1}$ is known but $g_{t:T}$ is not. If the algorithm is given the oracle knowledge that $g_{t:T}$ is sampled iid from a known distribution $\calD$, then the distribution of $\sum_{t=1}^Tg_t$ is also known. $x_t$ can then be assigned as the expectation of the best-in-hindsight decision, $x_t\leftarrow\E[\argmin_{u\in[-1,1]}\sum_{t=1}^Tg_tu]$. Tie-breaking of $\argmin$ is arbitrary assuming the density exists. 

The intuition of $\tilde x_t$ follows from this idea, but with two more ingredients. First, if CLT can be applied, then the distribution of $\sum_{t=1}^Tg_t$ is insensitive to $\calD$, as it can be approximated by $\calN(s_{t-1},T-t+1)$ as long as $T\gg t$. This embodies the universality principle in high-dimensional probability \citep{o2014analysis,van2014probability}. Second, the function $h$ introduces certain inductive bias, therefore the algorithm should minimize the regularized objective,
\begin{equation}\label{eq:regularized}
\tilde x_t=\E_{X\sim\calN(s_{t-1},T-t+1)}\spar{\argmin_{u\in[-1,1]}Xu+h^*(-u)}=-\E_{Z}[h'(s_{t-1}+\sqrt{T-t+1}Z)].
\end{equation}
Picking $h(s)=|s|$ gives a constant convex conjugate $h^*$, corresponding to the case without regularization. 

\paragraph{Continuous time dynamics} Another interpretation of $\tilde x_t$ is via the \emph{potential method} in online learning \citep{cesa2006prediction,orabona2025modern}. Let $\varphi_T:(-\infty,T]\times\R\rightarrow\R$ be the \emph{heat potential}
\begin{equation*}
\varphi_T(t,s)\defeq\E_Z[h(s+\sqrt{T-t}Z)],
\end{equation*}
which is the unique solution of the \emph{backward heat equation} (BHE) $(\partial_1+\half\partial_{22})\varphi_T=0$ under the terminal condition $\varphi_T(T,s)=h(s)$. It is straightforward to verify $\tilde x_t=-\partial_2\varphi_T(t-1,s_{t-1})$. This is the continuous time limit of the potential method, whose idea underlies a growing stream of recent works on better-than-big-O rates in OLO \citep{drenska2020prediction,kobzar2020a_new,greenstreet2022efficient,zhang2022pde,harvey2023optimal,harvey2024continuous,zhang2024improving}. It is also connected to Eq.\eqref{eq:regularized} via the equivalence between the potential method and the \emph{Follow the Regularized Leader} (FTRL) algorithm on linear losses \citep[Section~7.3]{orabona2025modern}. 

Going from $\tilde x_t$ to $x_t$, the key challenge is rigorously accounting for nonzero $c_t$. In particular, the bottleneck has been the \emph{discretization argument}: showing that a suitable modification of $\tilde x_t$ (equivalently, a numerical scheme on the BHE) works well against discrete time adversaries. Stein's method offers substantial advantages for this purpose. 

\paragraph{Computational complexity} By Lemma~\ref{lemma:equivalent_short}, $x_t$ can be implemented by a two-dimensional Gaussian quadrature whose time complexity is independent of $t$ and $T$. Assuming this computational subroutine is exact (which is the convention of the field), the conclusion is that Algorithm~\ref{algorithm:main} takes $O(1)$ time per round. In addition, the $h$ function of interest is sometimes simple enough such that $x_t$ can be explicitly reduced to common Gaussian-integral-type special functions. Examples include the error function and the \emph{Owen's T function}, for which fast and stable specialized implementations exist. This is demonstrated through concrete special cases in Appendix~\ref{section:reformulation}. 

\section{Analysis}\label{section:analysis}

The main result of this work is Theorem~\ref{theorem:main}, the total loss upper bound of Algorithm~\ref{algorithm:main}. An equivalent regret bound follows from the loss-regret duality (Corollary~\ref{corollary:regret} in Appendix~\ref{section:omitted_proofs}). Applications and extensions are presented in Section~\ref{section:dominate_ogd}, Appendix~\ref{section:dominate_mwu}, Appendix~\ref{section:soft_thresholded} and Appendix~\ref{section:stochastic}. 

\begin{theorem}[Main result; upper bound on $\loss_T$]\label{theorem:main}
Against all adversaries (potentially unbounded a priori), Algorithm~\ref{algorithm:main} given any convex and $1$-Lipschitz function $h$ guarantees
\begin{equation*}
\loss_T\leq \underbrace{-\E_{Z} \spar{h\rpar{\sum_{t=1}^Tg_t+\rho_T Z}}+\E_{Z}[h(\rho_0Z)]}_{\eqdef -\bar\psi^*_T(-\sum_{t=1}^Tg_t)}+\underbrace{\sum_{t=1}^T\rpar{\sqrt{\frac{2}{\pi}}\frac{\max\{g_t^2-c_t,0\}}{\rho_{t-1}}+\frac{2c_t\abs{g_t}+\abs{g_t^3}}{\rho_{t-1}^2}}}_{\eqdef \err_T}.
\end{equation*}
\end{theorem}

To parse this result, we divide the RHS into two parts, respectively defined as $-\bar\psi^*_T(-\sum_{t=1}^Tg_t)$ and $\err_T$. 
\begin{itemize}
\item $-\bar\psi^*_T(-\sum_{t=1}^Tg_t)$ is the primary component of the bound due to the continuous time dynamics of the OLO game. The most interpretable case is when $\rho_T=0$ and $h$ is even, meaning that the function $\bar\psi^*_T$ is a vertical shift of $h$. Generally, $\bar\psi^*_T$ mirrors $\psi^*_T$ from the targeted total loss bound Eq.\eqref{eq:loss_bound}: it is convex, $1$-Lipschitz, and analogous to Cover's achievability condition Eq.\eqref{eq:achievability_cover} we have $\E_{X\sim \calN(0,\rho_0^2-\rho_T^2)}[\bar\psi_T^*(X)]=0$. Achieving $\loss_T\leq -\bar\psi^*_T(-\sum_{t=1}^Tg_t)$ would be ideal, but as discussed previously, additional discretization error is necessary. Controlling this error has been the bottleneck in prior works. 
\item The key value of Theorem~\ref{theorem:main} is showing that the discretization error $\err_T$ of Algorithm~\ref{algorithm:main} can be controlled uniformly over all convex and Lipschitz $h$ functions. Moreover, $\err_T$ is a lower order term against many constrained adversary classes, thereby addressing the aforementioned bottleneck. For example, against the simple adversary class satisfying $\abs{g_t}\leq 1$ for all $t$, setting $h(x)=\abs{x}$ and $\rho_t=\sqrt{T-t}$ in Algorithm~\ref{algorithm:main} gives $\bar\psi^*_T(s)=\abs{s}-\sqrt{\frac{2}{\pi}T}$ for all $s\in\R$ while $\err_T=O(\log T)$. This corresponds to the particular performance tradeoff in Example~\ref{example:worst}; also see Section~\ref{subsection:absolute} for the equivalent argument on the regret. 
\end{itemize}

The following result shows that against adversaries satisfying $\abs{g_t}\leq 1$, Theorem~\ref{theorem:main} with $\rho_t=\sqrt{T-t}$ is optimal modulo the additive $O(\log T)$ factor. The proof is based on the standard nonasymptotic Wasserstein CLT for iid sums (e.g., \citep[Theorem~3.2]{ross2011fundamentals}), presented in Appendix~\ref{section:omitted_proofs}. 

\begin{restatable}[Lower bound on $\loss_T$]{theorem}{lossLower}\label{theorem:loss_lower}
There exists an absolute constant $c>0$ such that the following holds. For any OLO algorithm and any $1$-Lipschitz function $h:\R\rightarrow\R$, there exists a Boolean adversary ($g_t\in\{-1,1\}$ for all $t$) inducing
\begin{equation*}
\loss_T> -h\rpar{\sum_{t=1}^Tg_t}+\E_{Z}[h(\sqrt{T}Z)]-c.
\end{equation*}
\end{restatable}

\paragraph{Technical sketch} The crux of this work is the proof of Theorem~\ref{theorem:main}, which is short and considerably different from all known analyses in OLO. It is inspired by the work of R{\"o}llin \citep{rollin2018quantitative} combining Lindeberg's method and Stein's method to prove a quantitative (i.e., nonasymptotic) Wasserstein martingale CLT; that is, bounding the $1$-Wasserstein distance between a martingale and an appropriately scaled normal distribution. 

By its dual representation, the $1$-Wasserstein distance is an \emph{integral probability metric} over the class of $1$-Lipschitz functions. Fixing any function $h$ in this class, our key observation is that it is still valid to apply R{\"o}llin's argument on arbitrary data sequences generated by the OLO adversary. While such sequences are not martingales, the difference in the analysis is an additional bias term which is precisely the $\loss_T$ induced by Algorithm~\ref{algorithm:main}. In other words, Theorem~\ref{theorem:main} (and its generalization allowing randomized adversaries) may be viewed as an algorithmic analogue of R{\"o}llin's CLT, where the function $h$ associated with the algorithm serves as the \emph{witness function} measuring the statistical distance. This will be further discussed in Appendix~\ref{section:stochastic}. 

While the $1$-Wasserstein distance is defined over all $1$-Lipschitz functions, in Theorem~\ref{theorem:main} we also require $h$ to be convex, differing from R{\"o}llin's argument in the probabilistic context. This is motivated by the duality argument from Section~\ref{section:introduction}, and technically, it ensures that the function $f_{\mu,\sigma,h}$ is nonincreasing, leading to a low $\err_T$ through Eq.\eqref{eq:delta_convexity} in the proof. Nonconvex performance tradeoffs in OLO is generally a more delicate subject, which we further discuss in Appendix~\ref{section:nonconvex}.

\begin{proof}[Proof of Theorem~\ref{theorem:main}]
The proof is divided into the following steps. It uses several standard technical lemmas of Stein's method, provided in Appendix~\ref{section:background_stein}.

\paragraph{Step 1} The first step is the telescoping argument from Lindeberg's method. Let $Z_0,\ldots,Z_T$ be a sequence of iid standard normal random variables. For all $t\in[0:T]$, define
\begin{equation*}
m_t\defeq \rho_TZ_0+\sum_{i=1}^T\rpar{g_i\bm{1}[i\leq t]+\sqrt{c_i}Z_i\bm{1}[i> t]}\stackrel{d}{=}s_t+\rho_tZ_t,
\end{equation*}
where $\stackrel{d}{=}$ denotes the equality of marginal distributions (with respect to the randomness of $Z_{0:T}$). Starting from the telescopic sum $h(m_T)-h(m_0)=\sum_{t=1}^T\rpar{h(m_t)-h(m_{t-1})}$, we take the expectation with respect to $Z_{0:T}$ to obtain 
\begin{equation}
\E_{Z_T} \spar{h\rpar{\sum_{t=1}^Tg_t+\rho_T Z_T}}=\E_{Z_0}\spar{h(\rho_0Z_0)}+\sum_{t=1}^T\rpar{\E_{Z_t}\spar{h\rpar{s_t+\rho_tZ_t}}-\E_{Z_{t-1}}\spar{h\rpar{s_{t-1}+\rho_{t-1}Z_{t-1}}}}.\label{eq:telescopic}
\end{equation}
From this point the subscript of $Z$ is dropped. 

\paragraph{Step 2} The next step is using the Stein equation and \emph{Stein's lemma} (Lemma~\ref{lemma:stein_basic}) to simplify the expression under the sum for each $t\in[1:T]$. Since $f_{\mu,\sigma,h}$ satisfies the Stein equation Eq.\eqref{eq:stein}, 
\begin{equation*}
\rho_{t-1}^2f'_{s_{t-1},\rho_{t-1},h}(x)-(x-s_{t-1})f_{s_{t-1},\rho_{t-1},h}(x)=h(x)-\E_{Z}\spar{h(s_{t-1}+\rho_{t-1}Z)},\quad \forall x\in\R.
\end{equation*}
Plugging in $x\leftarrow s_t+\rho_tZ$ and taking the expectation with respect to $Z$, 
\begin{align*}
&\E_{Z}\spar{h(s_t+\rho_tZ)}-\E_{Z}\spar{h(s_{t-1}+\rho_{t-1}Z)}\\
=~&\E_{Z}\spar{\rho_{t-1}^2f'_{s_{t-1},\rho_{t-1},h}(s_t+\rho_tZ)-(s_t+\rho_tZ-s_{t-1})f_{s_{t-1},\rho_{t-1},h}(s_t+\rho_tZ)}\\
=~&\E_{Z}\spar{c_tf'_{s_{t-1},\rho_{t-1},h}(s_t+\rho_tZ)-g_tf_{s_{t-1},\rho_{t-1},h}(s_t+\rho_tZ)}\\
&\quad +\underbrace{\E_{Z}\spar{\rho_{t}^2f'_{s_{t-1},\rho_{t-1},h}(s_t+\rho_tZ)-\rho_tZf_{s_{t-1},\rho_{t-1},h}(s_t+\rho_tZ)}}_{=0}.\tag{Lemma~\ref{lemma:stein_basic}}
\end{align*}
On the last line, the integrability assumption required by Lemma~\ref{lemma:stein_basic} follows from Lemma~\ref{lemma:stein_magnitude} (Stein factors).

\paragraph{Step 3} The final step is plugging in Stein factors and exploiting the convexity of $h$. Recall that $f'_{s_{t-1},\rho_{t-1},h}$ is absolutely continuous. By Taylor's theorem with integral remainder, 
\begin{align*}
f'_{s_{t-1},\rho_{t-1},h}(s_t+\rho_tZ)&=f'_{s_{t-1},\rho_{t-1},h}(s_{t-1}+\rho_tZ)+g_t\int_0^1f''_{s_{t-1},\rho_{t-1},h}(s_{t-1}+\rho_tZ+\lambda g_t)\diff\lambda,\\
f_{s_{t-1},\rho_{t-1},h}(s_t+\rho_tZ)&=f_{s_{t-1},\rho_{t-1},h}(s_{t-1}+\rho_tZ)+g_t f'_{s_{t-1},\rho_{t-1},h}(s_{t-1}+\rho_tZ)\\
&\quad +g_t^2\int_0^1(1-\lambda)f''_{s_{t-1},\rho_{t-1},h}(s_{t-1}+\rho_tZ+\lambda g_t)\diff\lambda.
\end{align*}
Therefore, combining the above, 
\begin{align*}
&\E_{Z}\spar{h(s_t+\rho_tZ)}-\E_{Z}\spar{h(s_{t-1}+\rho_{t-1}Z)}\\
=~&-g_t\underbrace{\E_{Z}\spar{f_{s_{t-1},\rho_{t-1},h}(s_{t-1}+\rho_tZ)}}_{=x_{t}} +\underbrace{(c_t-g_t^2)\E_{Z}\spar{f'_{s_{t-1},\rho_{t-1},h}(s_{t-1}+\rho_tZ)}}_{\eqdef\triangle}\\
&+\underbrace{\E_{Z}\spar{c_t g_t\int_0^1f''_{s_{t-1},\rho_{t-1},h}(s_{t-1}+\rho_tZ+\lambda g_t)\diff\lambda-g_t^3\int_0^1(1-\lambda)f''_{s_{t-1},\rho_{t-1},h}(s_{t-1}+\rho_tZ+\lambda g_t)\diff\lambda}}_{\eqdef\Diamond}.
\end{align*}
For the terms $\triangle$ and $\Diamond$ on the RHS, we use Lemma~\ref{lemma:stein_magnitude} (Stein factors) and Lemma~\ref{lemma:monotonicity} ($f'_{s_{t-1},\rho_{t-1},h}$ is non-positive), 
\begin{align}
\triangle&\leq \sqrt{\frac{2}{\pi}}\rho^{-1}_{t-1}\cdot\max\{g_t^2-c_t,0\},\label{eq:delta_convexity}\\
\Diamond&\leq \frac{2c_t\abs{g_t}}{\rho_{t-1}^2}+\frac{2\abs{g_t^3}}{\rho_{t-1}^2}\int_0^1(1-\lambda)\diff\lambda= \frac{2c_t\abs{g_t}+\abs{g_t^3}}{\rho_{t-1}^2}.\nonumber
\end{align}

Combining all three steps completes the proof.
\end{proof}

\section{Application: Dominating OGD}\label{section:dominate_ogd}

This section instantiates our results with specific $h$ functions. We will assume $\abs{g_t}\leq 1$ for simplicity, therefore setting $\rho_t=\sqrt{T-t}$ gives $\err_T=O(\log T)$. Proofs are deferred to Appendix~\ref{section:omitted_proofs}, and the closed forms of the output $x_t$ are presented in Appendix~\ref{section:reformulation}. 

\subsection{Absolute Value Function}\label{subsection:absolute}

As a warm-up, picking $h(x)=\abs{x}$ corresponds to minimizing the uniform regret in Example~\ref{example:worst}. The following regret bound matches existing ones \citep{kobzar2020a_new,greenstreet2022efficient} in the optimal leading order term $\sqrt{\frac{2}{\pi}T}$.  

\begin{restatable}[Regret: absolute value]{corollary}{absoluteProp}\label{corollary:regret_absolute}
Assume $\abs{g_t}\leq 1$ for all $t$. If $h(x)=\abs{x}$, then Algorithm~\ref{algorithm:main} with $\rho_t=\sqrt{T-t}$ guarantees
\begin{equation*}
\reg_T(u)\leq \reg_T^\unif\leq \sqrt{\frac{2}{\pi}T}+O(\log T),\quad\forall u\in[-1,1].
\end{equation*}
\end{restatable}

\subsection{Huber Function}\label{subsection:huber}

A highlight of this work is picking $h$ as the \emph{Huber function}, whose inductive bias matches that of the OGD baseline. The following regret bound generalizes Corollary~\ref{corollary:regret_absolute} in the limit of $\eta\rightarrow\infty$. By the regret version of Theorem~\ref{theorem:loss_lower} (Corollary~\ref{corollary:regret_lower} in Appendix~\ref{section:omitted_proofs}), it cannot be improved by $\Omega(\log T)$ uniformly over $u\in[-1,1]$. 

\begin{restatable}[Regret: Huber]{corollary}{huberProp}\label{corollary:regret_huber}
Assume $\abs{g_t}\leq 1$ for all $t$. If $h(x)=\frac{\eta}{2}x^2\cdot\bm{1}[\abs{x}\leq \eta^{-1}]+(\abs{x}-\frac{1}{2\eta})\cdot\bm{1}[\abs{x}> \eta^{-1}]$ for $\eta=\frac{\alpha}{\sqrt{T}}$ and $\alpha\in\R_{>0}$, then Algorithm~\ref{algorithm:main} with $\rho_t=\sqrt{T-t}$ guarantees
\begin{equation*}
\reg_T(u)\leq\underbrace{\spar{\frac{u^2}{2\alpha}+\rpar{\alpha+\frac{1}{\alpha}} \Phi\rpar{\frac{1}{\alpha}}+\phi\rpar{\frac{1}{\alpha}}-\half\alpha-\frac{1}{\alpha}}}_{\eqdef \gamma_{\mathrm{Huber}}(u,\alpha)}\sqrt{T}+O(\log T),\quad\forall u\in[-1,1].
\end{equation*}
\end{restatable}

Since the algorithm takes $O(1)$ time per round, we compare it to the OGD baseline with constant learning rate $\eta=\frac{\alpha}{\sqrt{T}}$ which outputs $x_t=\Pi_{[-1,1]}(x_{t-1}-\eta g_{t-1})$. Against the same adversary class, the OGD baseline guarantees $\reg_T(u)\leq \half(\alpha^{-1}u+\alpha)\sqrt{T}$ \cite[Theorem~2.13]{orabona2025modern}, where the prefactor of $\sqrt{T}$ will be denoted as $\gamma_{\mathrm{OGD}}(u,\alpha)$. To our knowledge, this is the sharpest known regret bound of OGD that depends on $u$, $\alpha$ and $T$. As $O(\log T)$ in Corollary~\ref{corollary:regret_huber} is a lower order term, we focus on the prefactors of $\sqrt{T}$ and define the \emph{margin of improvement} as
\begin{equation*}
\mathrm{Gap}_{\mathrm{OGD}}(\alpha)\defeq\gamma_{\mathrm{OGD}}(u,\alpha)-\gamma_{\mathrm{Huber}}(u,\alpha)=\rpar{\alpha+\frac{1}{\alpha}}\spar{1-\Phi\rpar{\frac{1}{\alpha}}}-\phi\rpar{\frac{1}{\alpha}}.
\end{equation*}
The following observations can be made; also see Figure~\ref{figure:ogd} in Section~\ref{section:introduction} for visualizations.
\begin{itemize}
\item Consistent with its notation, $\mathrm{Gap}_{\mathrm{OGD}}(\alpha)$ is independent of $u$, and it is always positive due to a standard \emph{Mills ratio} estimate (Lemma~\ref{lemma:mills}). Therefore Corollary~\ref{corollary:regret_huber} dominates the regret bound of OGD for all $u\in[-1,1]$, while they are both incomparable to sharp uniform regret bounds, including our Corollary~\ref{corollary:regret_absolute} and the results of \citep{kobzar2020a_new,greenstreet2022efficient}. By Eq.\eqref{eq:loss_bound}, this ensures that the $\loss_T$ upper bound of our algorithm is lower than that of OGD regardless of $\sum_{t=1}^Tg_t$ -- that is, regardless of the difficulty of the adversary. 

\item $\mathrm{Gap}_{\mathrm{OGD}}(\alpha)$ grows unboundedly with $\alpha$. For all $u\in[-1,1]$, $\gamma_{\mathrm{OGD}}(u,\alpha)$ diverges as $\alpha\rightarrow\infty$, while
\begin{equation*}
\lim_{\alpha\rightarrow\infty}\gamma_{\mathrm{Huber}}(u,\alpha)=\lim_{\alpha\rightarrow\infty}\alpha\spar{\Phi\rpar{\frac{1}{\alpha}}-\Phi(0)}+\phi\rpar{\frac{1}{\alpha}}=2\phi(0)=\sqrt{\frac{2}{\pi}},
\end{equation*}
recovering Corollary~\ref{corollary:regret_absolute}. 

\item To reason about the ``effective learning rate'' of our algorithm, consider $\tilde x_t$ from Eq.\eqref{eq:regularized} which is the small-$c_t$ approximation of the output $x_t$. Suppose for simplicity that $s_{t-1}=0$ for some $t\in[1:T]$. By symmetry, we have $\tilde x_t=0$, meaning that $\tilde x_{t+1}$ only depends on $g_t$, $\eta$ and $\rho_t$. Linearizing it with respect to $g_t$ gives (see Remark~\ref{remark:effective_lr} for derivation)
\begin{equation*}
\tilde x_{t+1}-\tilde x_t=-\eta\cdot \erf\rpar{\frac{1}{\sqrt{2}\eta\rho_t}}g_t+o(g_t).
\end{equation*}
Treating the absolute value of the linear coefficient as the effective learning rate, we see that it is increasing with respect to $t$ and eventually reaches the nominal value $\eta$, the learning rate of the OGD baseline. This is different from typical OLO algorithms where the learning rate is constant (if $T$ is known) or decreasing (if $T$ is unknown). 
\end{itemize}

Other interesting cases of $h$ are analyzed in the appendix. 

\section{Conclusion}\label{section:conclusion}

Utilizing Stein's method, this paper presents a general, sharp and analytically scalable framework to design and analyze OLO algorithms. Inspired by R{\"o}llin's proof of a Wasserstein martingale CLT \citep{rollin2018quantitative}, our main result is a computationally efficient improvement to a seminal dynamic programming algorithm by Cover \citep{cover1966behavior}. Starting from there, we demonstrate a number of more specific quantitative benefits, including ($i$) dominating the total loss upper bounds of OGD and MWU; ($ii$) achieving optimal tradeoffs between the total loss and the uniform regret; and ($iii$) handling OLO with noisy feedback. Together, they imply an intriguing algorithmic relation between adversarial online learning and probabilistic limit theorems, whose applicability might extend much further. 

This work only studies the basic setting of OLO with a one-dimensional bounded domain and a known time horizon. Generalizing such conditions (that is, achieving additively sharp rates while allowing general dimension, unknown time horizon or unbounded domain) would be a highly valuable direction for future works. Other than OLO, the idea of universality may also benefit online learning problems with curved losses or bandit feedback, where little is known about better-than-big-O optimality. Finally, besides giving improvements to existing problems, we hope the connection to probabilistic limit theorems could help motivate and answer new problems in adversarial online learning, such as operationalizing the Kolmogorov distance versions of the martingale CLT. 

\section*{Acknowledgment}

This work is assisted by AI tools for literature search, preliminary calculation and writing. The authors take full responsibility for the originality and correctness of the results. 

\bibliography{Stein_OL}

\newpage
\appendix
\section*{Appendix}

\paragraph{Organization} Appendix~\ref{section:related} discusses related works, followed by the preliminaries of Stein's method in Appendix~\ref{section:background_stein}. Appendix~\ref{section:dominate_mwu} continues the argument in Section~\ref{section:dominate_ogd}, showing that a special case of Algorithm~\ref{algorithm:main} improves upon MWU. Appendix~\ref{section:soft_thresholded} shows that Algorithm~\ref{algorithm:main} can attain a continuum of optimal two-point tradeoffs between $\loss_T$ and $\reg^\unif_T$. Appendix~\ref{section:stochastic} extends our results to OLO with noisy feedback and further compares our result to R{\"o}llin's martingale CLT. Appendix~\ref{section:nonconvex} discusses the generalization to nonconvex performance tradeoffs in OLO. Appendix~\ref{section:omitted_proofs} contains omitted proofs. Finally, Appendix~\ref{section:reformulation} presents the closed-form expressions of $f_{\mu,\sigma,h}$ and $x_t$ in special cases, thus improving the concreteness of our algorithm. 

\section{Related Works}\label{section:related}

From a quantitative perspective, our work is related to two general research themes in online learning: ($i$) achieving leading-constant-optimal upper bounds on the uniform regret; and ($ii$) achieving comparator-dependent regret bounds. Besides those, from a methodological perspective, our work touches upon the deep connection between adversarial online learning and martingale inequalities. The latter has been studied by a line of works whose objectives and techniques are very different from ours. 

\paragraph{Sharp uniform regret bounds} Consider the uniform regret $\reg^\unif_T$ from Example~\ref{example:worst} which is the maximum of the regret over all comparators. While most works focus on the conventional big-O optimality when bounding the $\reg^\unif_T$, here we note several exceptional settings where sharper bounds have been achieved. Specifically for $\reg^\unif_T$, this precisely means achieving the optimal constant multiplier on the leading order term of the bound, without being subsumed by big O. 

As discussed in Section~\ref{subsection:motivation}, for the known-horizon version of OLO on a one-dimensional bounded domain (i.e., the setting of this work), \citep{cover1966behavior} gave the first algorithm achieving the optimal leading constant, and computational efficient improvements were later given by \citep{kobzar2020a_new,greenstreet2022efficient}. Generalizing the domain to higher dimensions while still achieving the leading constant optimality is however a delicate subject requiring a careful treatment of the shape of the domain. 
\begin{itemize}
\item The simple case is the $d$-dimensional Euclidean norm ball: for any $d\geq 2$, \citep{abernethy2008optimal} showed that OGD with constant learning rate achieves the optimal constant multiplier in front of the leading order term $\sqrt{T}$. 
\item The difficult case is the $d$-dimensional probability simplex (that is, the LEA problem with $d$ experts): MWU with constant learning rate is leading-constant-optimal in the regime of both $d\rightarrow\infty$ and $T\rightarrow\infty$ \citep[Theorem~3.7]{cesa2006prediction}. Only taking $T\rightarrow\infty$, precisely leading-constant-optimal algorithms for $d=2,3,4$ are give by \citep{cover1966behavior}, \citep{gravin2016towards} and \citep{bayraktar2020finite} respectively. The regime in between remains open. 
\end{itemize}
In a different direction, \citep{harvey2023optimal} achieved leading constant optimality for the \emph{anytime} version of OLO (i.e., $T$ is unknown by the algorithm) on a one-dimensional bounded domain. A common technical theme shared by these recent progresses is a continuous time perspective on the OLO problem, pioneered by \citep{zhu2014two,drenska2020prediction,harvey2023optimal} and realized by tools from PDE and stochastic analysis. The bottleneck is the discretization argument on the PDE, discussed in Section~\ref{section:algorithm}, for which Stein's method offers analytical advantages. 

\paragraph{Parameter-free OLO} Going beyond the uniform regret, there is a subfield within OLO studying the realization of $u$-dependent upper bounds on $\reg_T(u)$. Depending on the context, this is often called \emph{parameter-free online learning} or \emph{comparator-adaptive online learning}. See \citep[Section~10]{orabona2025modern} for an exposition. 

A key motivation of this topic originates from the practical need to handle unbounded domains. In one dimension, it means the domain of the algorithm is $\R$ as opposed to $[-1,1]$. The definition of $\reg_T^\unif$ becomes vacuous, therefore the typical quantitative objective is to achieve a $u$-dependent upper bound on $\reg_T(u)$ whose dependence on both $\abs{u}\rightarrow\infty$ and $T\rightarrow\infty$ is asymptotically optimal. Based on a similar technical backbone as \citep{harvey2023optimal}, an algorithm with the optimal leading constant is given by \citep{zhang2022pde}. Such results are also meaningful for the settings with a bounded domain, due to a projection technique of \citep{cutkosky2018black} showing that if a particular $u$-dependent regret bound is achievable in the unconstrained setting for all $u\in\R$, then it is also achievable in the constrained setting for all $u$ within the domain of the algorithm. 

Compared to the state of the art in parameter-free OLO, our work has a subtle but important difference: rather than considering both $\abs{u}\rightarrow\infty$ and $T\rightarrow\infty$, we keep $u$ fixed on a bounded domain and only consider the asymptotic regime of $T\rightarrow\infty$. Via the loss-regret duality (Lemma~\ref{lemma:duality}), this is closely related to the objectives of \citep{cover1966behavior}, which we use to motivate our techniques. Appendix~\ref{section:soft_thresholded} shows that an application of such techniques improves upon the combination of \citep{zhang2022pde} and \citep{cutkosky2018black}. 

\paragraph{Connection to martingale concentration} Conceptually, our work can be viewed as another evidence of the deep connection between adversarial online learning and martingale inequalities. In a broader context, such a connection has been studied by many others, including \citep{rakhlin2017equivalence,foster2018online,jun2019parameter,orabona2023tight,waudby2024estimating,agrawal2025eventually}, just to name a few. There is also an emerging field called game-theoretic probability \citep{shafer2019game} dedicated to the systematic study of this connection. 

Our work differs from the mentioned works in several major aspects. 
The most notable one is that while these works have characterized various intriguing relations between game-theoretic regret bounds and martingale \emph{concentration inequalities} (e.g., exponential tail probability bounds), we focus on such a relation with martingale CLTs on the probabilistic side. This allows us to obtain OLO algorithms with additively sharp rates which is beyond the reach of existing concentration-based arguments; for example, the seminal work of \citep{rakhlin2017equivalence} proposed a conversion from martingale concentration inequalities to the achievability of certain regret bounds, which, as noted by the authors, pays a price through the multiplying constants. Other differences include: ($i$) using Stein's method in adversarial online learning has not been considered by prior works; and ($ii$) we emphasize the somewhat unusual perspective of total loss minimization in OLO which includes uniform regret minimization as a special case, while the latter is often the starting point of the works mentioned above. 


\section{Preliminaries of Stein's Method}\label{section:background_stein}

This section presents the preliminaries of Stein's method, based on \citep{chen2010normal,nourdin2012normal}. 

\paragraph{Stein's lemma} The motivation of Stein's method comes from the celebrated Stein's lemma. We only need its special case for one-dimensional normal distributions \citep[Lemma~2.1]{chen2010normal}. 

\begin{lemma}[Stein's lemma]\label{lemma:stein_basic}
Let $X\sim\calN(\mu,\sigma^2)$. For all absolutely continuous function $f:\R\rightarrow\R$ such that $\E[\abs{f'(X)}]<\infty$, we have
\begin{equation*}
\sigma^2\E[f'(X)]=\E[(X-\mu)f(X)].
\end{equation*}
\end{lemma}

\paragraph{Solution of Stein equation} Consider the Stein equation Eq.\eqref{eq:stein} with a convex $1$-Lipschitz function $h$, 
\begin{equation*}
\sigma^2f'(x)-(x-\mu)f(x)=h(x)-\E_Z[h(\mu+\sigma Z)],\quad \forall x\in\R. 
\end{equation*}
With a change of variable from the special case of $\mu=0$ and $\sigma=1$ \citep[Proposition~3.2.2 and Remark~3.2.4]{nourdin2012normal}, it is a standard fact that the unique bounded solution $f_{\mu,\sigma,h}$ is given by the two equivalent expressions,
\begin{align}
f_{\mu,\sigma,h}(x)&=\frac{1}{\sigma^2\phi_{\mu,\sigma}(x)}\int_{-\infty}^x(h(z)-\E_Z[h(\mu+\sigma Z)])\phi_{\mu,\sigma}(z)\diff z\label{eq:standard_rep}\\
&=-\frac{1}{\sigma^2\phi_{\mu,\sigma}(x)}\int_{x}^{\infty}(h(z)-\E_Z[h(\mu+\sigma Z)])\phi_{\mu,\sigma}(z)\diff z.\nonumber
\end{align}
Besides, the solution $f_{\mu,\sigma,h}$ admits another equivalent representation through the Ornstein–Uhlenbeck (OU) semigroup $P_t$,
\begin{equation*}
(P_th)(x)\defeq \E_Z\spar{h\rpar{\mu+e^{-t}(x-\mu)+\sigma\sqrt{1-e^{-2t}}Z}},
\end{equation*}
and since $h$ is differentiable almost everywhere,
\begin{align}
f_{\mu,\sigma,h}(x)&=-\frac{\partial}{\partial_x}\int_0^\infty \rpar{(P_\tau h)(x)-\E_Z[h(\mu+\sigma Z)]}\diff \tau\nonumber\\
&=-\int_0^\infty \frac{\partial}{\partial_x}(P_\tau h)(x)\diff \tau\nonumber\\
&=-\int_0^\infty e^{-\tau}\E_Z\spar{h'\rpar{\mu+e^{-\tau}(x-\mu)+\sigma\sqrt{1-e^{-2\tau}}Z}}\diff \tau.\label{eq:ou_rep}
\end{align}
This is a change of variable from \citep[Proposition~3.5.1]{nourdin2012normal}. 

\paragraph{Stein factors} When restricted to the class of Lipschitz $h$ functions, Stein factors refer to bounds on the essential supremum norm of $f_{\mu,\sigma,h}$, $f'_{\mu,\sigma,h}$ and $f''_{\mu,\sigma,h}$ via the Lipschitz constant of $h$. Here we present the version from \citep{gaunt2025stein} with a change of variable. Besides being useful for error control, it also certifies the integrability requirement of Lemma~\ref{lemma:stein_basic} in our analysis. 

\begin{lemma}[Stein factors]\label{lemma:stein_magnitude}
For all $1$-Lipschitz function $h$, 
\begin{equation*}
\norm{f_{\mu,\sigma,h}}_\infty\leq 1,\quad \norm{f'_{\mu,\sigma,h}}_\infty\leq \sqrt{\frac{2}{\pi}}\sigma^{-1},\quad \norm{f''_{\mu,\sigma,h}}_\infty\leq 2\sigma^{-2}.
\end{equation*}
\end{lemma}

\paragraph{Monotonicity} Finally, 
the following result states the monotonicity of $f_{\mu,\sigma,h}$ when the function $h$ is convex and Lipschitz. This is a straightforward observation from the semigroup representation Eq.\eqref{eq:ou_rep} as well as the fact that $f_{\mu,\sigma,h}\in C^1$. 

\begin{lemma}[Monotonicity]\label{lemma:monotonicity}
For all convex $1$-Lipschitz function $h$, $f'_{\mu,\sigma,h}(x)\leq 0$ for all $x\in\R$.
\end{lemma}

\section{Application: Dominating MWU}\label{section:dominate_mwu}

This section further develops the argument from Section~\ref{section:dominate_ogd}, showing that another special case of our algorithm improves upon the MWU baseline. In particular, it requires setting the function $h$ in Algorithm~\ref{algorithm:main} as the \emph{log-sum-exp function}. The following $\reg_T(u)$ upper bound is a corollary of Theorem~\ref{theorem:main}, and by Corollary~\ref{corollary:regret_lower}, it cannot be improved by $\Omega(\log T)$ uniformly over $u\in[-1,1]$. 

\begin{restatable}[Regret: log-sum-exp]{corollary}{lseProp}\label{corollary:regret_logcosh}
Assume $\abs{g_t}\leq 1$ for all $t$. If $h(x)=\frac{1}{\eta}\ln(\cosh(\eta x))$ for $\eta=\frac{\alpha}{\sqrt{T}}$ and $\alpha\in\R_{>0}$, then Algorithm~\ref{algorithm:main} with $\rho_t=\sqrt{T-t}$ guarantees
\begin{equation*}
\reg_T(u)\leq \underbrace{\spar{\frac{1}{2\alpha}\ln\rpar{(1+u)^{1+u}(1-u)^{1-u}}+\E_{Z}\spar{\alpha^{-1}\ln(\cosh(\alpha Z))}}}_{\eqdef\gamma_{\mathrm{LSE}}(u,\alpha)}\sqrt{T}+O(\log T),\quad\forall u\in[-1,1].
\end{equation*}
In particular, the RHS follows the convention $0^0=1$ when $u=\pm 1$. 
\end{restatable}

While the MWU baseline has been mostly studied in the so-called ``expert problem'', its special case with two experts can be equivalently adapted to our setting. This is a well-known fact, and we provide the details in Appendix~\ref{subsection:detail_lse} for concreteness. With the constant learning rate $\eta=\frac{\alpha}{\sqrt{T}}$, the obtained OLO algorithm guarantees \citep[Section~7.5]{orabona2025modern}
\begin{equation}\label{eq:mwu}
\reg_T(u)\leq \underbrace{\spar{\frac{1}{2\alpha}\ln\rpar{(1+u)^{1+u}(1-u)^{1-u}}+\frac{1}{2}\alpha}}_{\eqdef\gamma_{\mathrm{MWU}}(u,\alpha)}\sqrt{T}.
\end{equation}
Relaxing this bound gives $\reg^\unif_T\leq [(\ln 2) \frac{1}{\alpha}+\half\alpha]\sqrt{T}$ which is tight for the MWU baseline, as \citep{gravin2017tight} showed that MWU must suffer $\reg^\unif_T\geq \sqrt{2(\ln 2)T}$ even when non-increasing (as opposed to constant) learning rates are allowed. 

\begin{figure}[t]
\includegraphics[width=\textwidth]{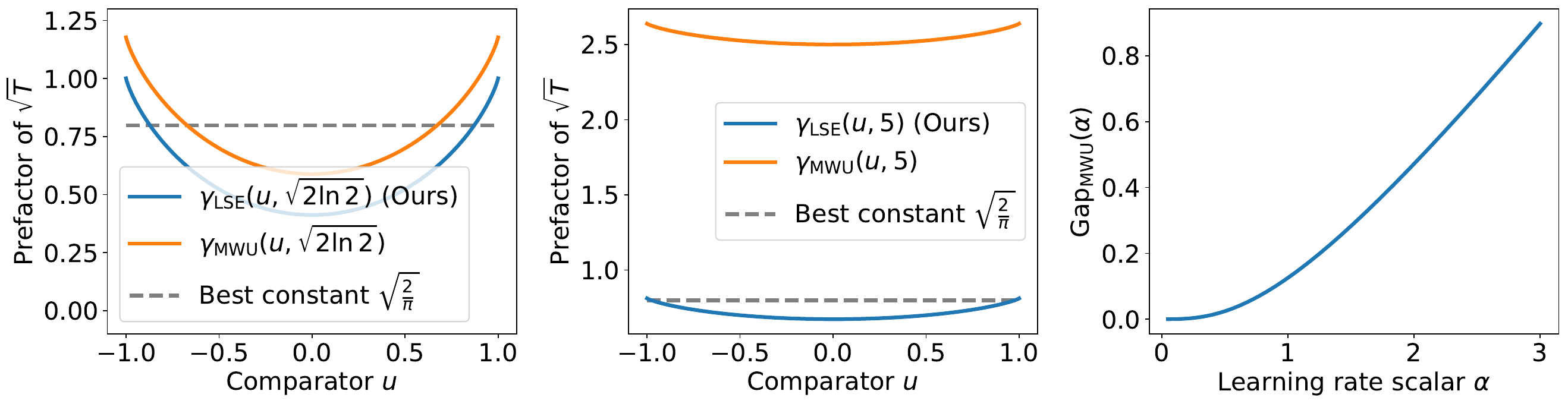}
\caption{Comparison of the prefactors of $\sqrt{T}$ in the regret bounds of our algorithm (represented by $\gamma_{\mathrm{LSE}}(u,\alpha)$), OGD (represented by $\gamma_{\mathrm{MWU}}(u,\alpha)$), and Cover's algorithm (the optimal $u$-independent prefactor $\sqrt{\frac{2}{\pi}}$). Analogous to Figure~\ref{figure:ogd}, but based on the log-sum-exp regime (Corollary~\ref{corollary:regret_logcosh}). Left: with $\alpha=\sqrt{2\ln 2}$ which minimizes $\sup_{u\in[-1,1]}\gamma_{\mathrm{MWU}}(1,\alpha)$, $\gamma_{\mathrm{LSE}}(u,\alpha)$ and $\gamma_{\mathrm{MWU}}(u,\alpha)$ are compared as functions of $u$, lower is better. Middle:  with $\alpha=5$, the gap between $\gamma_{\mathrm{MWU}}(u,\alpha)$ and $\gamma_{\mathrm{LSE}}(u,\alpha)$ widens. Right: the margin of improvement $\mathrm{Gap}_{\mathrm{MWU}}(\alpha)\defeq\gamma_{\mathrm{MWU}}(u,\alpha)-\gamma_{\mathrm{LSE}}(u,\alpha)$ as a function of $\alpha$.
}
\label{figure:mwu}
\end{figure}

The insight from Section~\ref{section:dominate_ogd} carries over to the comparison between Corollary~\ref{corollary:regret_logcosh} and Eq.\eqref{eq:mwu}. Specifically, we define the $u$-independent $\mathrm{Gap}_{\mathrm{MWU}}(\alpha)\defeq\gamma_{\mathrm{MWU}}(u,\alpha)-\gamma_{\mathrm{LSE}}(u,\alpha)$ as the difference between the two prefactors of $\sqrt{T}$. Lemma~\ref{lemma:expectation_lse} shows that $\mathrm{Gap}_{\mathrm{MWU}}(\alpha)> 0$ for all $\alpha>0$, and $\lim_{\alpha\rightarrow\infty}\gamma_{\mathrm{LSE}}(u,\alpha)=\sqrt{\frac{2}{\pi}}$ while $\gamma_{\mathrm{MWU}}(u,\alpha)$ diverges. Analogous to Figure~\ref{figure:ogd} in the main paper, results from this regime are visualized in Figure~\ref{figure:mwu}. Combining it with the lower bound of \citep{gravin2017tight} gives a quantitative separation between our algorithm and MWU. 

\subsection{Details of the Baseline}\label{subsection:detail_lse}

\paragraph{LEA and the MWU algorithm} The MWU baseline was originally designed to solve the problem of \emph{distributional LEA}, with the following definition. Consider a repeated game lasting $T$ rounds. In each round, 
\begin{itemize}
\item The LEA algorithm picks the decision $w_t\in\Delta_d$, where $\Delta_d=\{w\in\R^d:\norm{w}_1=1, w_i\geq 0,\forall i\in[1:d]\}$ represents the probability simplex embedded in $\R^d$.
\item The adversary picks the feedback $l_t\in\R^d$ satisfying $\norm{l_t}_\infty\leq 1$. 
\item The LEA algorithm suffers the loss $\inner{l_t}{w_t}$. 
\end{itemize}
At the end of the game, the performance of the algorithm is measured by the \emph{LEA regret}: for any comparator $\tilde u\in\Delta_d$,
\begin{equation*}
\reg^{\mathrm{LEA}}_T(\tilde u)\defeq\sum_{t=1}^T\inner{l_t}{w_t-\tilde u}. 
\end{equation*}

Our work is only relevant to the case of $d=2$. In this case, with the learning rate $\eta=\frac{\alpha}{\sqrt{T}}$ where $\alpha\in\R_{>0}$, the MWU algorithm from \citep[Section~7.5]{orabona2025modern} outputs
\begin{equation*}
x_{t,j}=\frac{\exp\rpar{-\eta\sum_{\tau=1}^{t-1}l_{\tau,j}}}{\sum_{i=1}^d\exp\rpar{-\eta\sum_{\tau=1}^{\tau-1}l_{k,i}}},\quad \forall j\in[1:2],
\end{equation*}
and guarantees the LEA regret bound (typically stated using the KL divergence)
\begin{equation}\label{eq:mwu_bound}
\reg^{\mathrm{LEA}}_T(\tilde u)\leq \spar{\frac{1}{\alpha}\rpar{\tilde u_1\ln \tilde u_1+\tilde u_2\ln \tilde u_2+\ln 2}+\half \alpha}\sqrt{T},\quad \forall \tilde u\in\Delta_2.
\end{equation}

\paragraph{Converting LEA to 1D OLO} With the above MWU algorithm denoted as $\A$, we use it to define the following OLO algorithm for the setting we consider. In each round, 
\begin{itemize}
\item Query $\A$ for its output $w_t\in\Delta_2$, and then output $x_t=w_{t,1}-w_{t,2}\in[-1,1]$ in one-dimensional OLO. 
\item After receiving the loss gradient $g_t\in\R$ in OLO, define $l_t=[g_t,-g_t]$ and send it to $\A$ as its $t$-th feedback in LEA. As long as $\abs{g_t}\leq 1$, it ensures $\norm{l_t}_{\infty}\leq 1$. 
\end{itemize}
Furthermore, for any OLO comparator $u\in[-1,1]$, define an LEA comparator $\tilde u$ by $\tilde u_1=\half(1+u)$ and $\tilde u_2=\half(1-u)$. These procedures ensure that for all $t$ we have $g_tx_t=\inner{l_t}{w_t}$, and the total loss of the comparators satisfies $\sum_{t=1}^Tg_tu=\sum_{t=1}^T\inner{l_t}{\tilde u}$. The point is that any achievable upper bound on $\reg^{\mathrm{LEA}}_T(\tilde u)$ can also be realized in the OLO problem we consider, and plugging in Eq.\eqref{eq:mwu_bound} gives Eq.\eqref{eq:mwu}.

\section{Application: Trading off Total Loss and Uniform Regret}\label{section:soft_thresholded}

This section presents the second of the three applications summarized in Section~\ref{subsection:results}: Algorithm~\ref{algorithm:main} can achieve a continuum of optimal two-point tradeoffs between the total loss and the uniform regret. 

The motivation is the following. Our main result (Theorem~\ref{theorem:main}) shows that given any convex and $1$-Lipschitz $h$ function, Algorithm~\ref{algorithm:main} guarantees additively sharp regret and total loss upper bounds, each realizing a performance tradeoff specified by $h$. While different $h$ functions are generally incomparable, a natural follow-up question is whether one could characterize the optimal $h$ given a reasonable preference over all possible tradeoffs. The two-point tradeoff we consider corresponds to such a preference. 

\paragraph{Setting} Starting from the necessary notations, let $\mathrm{alg}$ represent the OLO algorithm and $\mathrm{adv}$ represent the adversary. Throughout this section the adversary $\mathrm{adv}$ is assumed to satisfy $g_t\in[-1,1]$ for all $t$. Fixing the algorithm and the adversary, both the total loss $\loss_T$ and the uniform regret $\reg_T^\unif$ are determined, therefore we overload such notations by $\loss_T(\mathrm{alg},\mathrm{adv})$ and $\reg^\unif_T(\mathrm{alg},\mathrm{adv})$ to clarify their dependencies. Similarly, the notation $\reg_T(u)$ is overloaded by $\reg_T(u,\mathrm{alg},\mathrm{adv})$.

Given any constant $\eps\in(0,\sqrt{\frac{\pi}{2}}]$ (independent of $T$), we are interested in characterizing
\begin{equation}\label{eq:gamma}
\gamma(\eps)\defeq\limsup_{T\rightarrow\infty}\frac{1}{\sqrt{T}}\inf_{\mathrm{alg}}\left\{\sup_{\mathrm{adv}}\reg_T^\unif(\mathrm{alg},\mathrm{adv});\sup_{\mathrm{adv}}\loss_T(\mathrm{alg},\mathrm{adv})\leq\eps\sqrt{T}\right\}.
\end{equation}
Here we emphasize that the supremum over $\mathrm{adv}$ needs to respect the boundedness assumption $g_t\in[-1,1]$. The interpretation is that up to an additive $o(\sqrt{T})$ error, the optimal uniform regret bound subject to the constraint $\loss_T\leq \eps\sqrt{T}$ is $\reg^\unif_T\leq \gamma(\eps)\sqrt{T}$. Furthermore, we aim to design a computationally efficient algorithm achieving these performance guarantees. 

\paragraph{Background} Two-point tradeoffs of this flavor are ubiquitous in statistics and machine learning, such as the tradeoff between Type-I and Type-II errors in hypothesis testing. The pioneering works of \citep{even2008regret,kapralov2011prediction} initiated this study in OLO, although the emphasis is on the big-O optimality rather than the leading constant optimality in a different quantitative regime ($\eps\propto T^{-1/2}$ as opposed to constant $\eps$; also see Appendix~\ref{section:related}). Interestingly, the algorithm of \citep{kapralov2011prediction} is based on a ``Stein-like'' ODE, while it appears that the authors might not be aware of this connection at the time therefore the analysis was by brute force and suboptimal. Without considering computational efficiency, \citep{koolen2013pareto} characterized a similar optimality frontier in the two-expert LEA problem, by passing a Cover-style achievability condition to the $T\rightarrow\infty$ limit. A general strategy for trading off $\loss_T$ with $\reg_T^\unif$ is via the combination of an unconstrained parameter-free OLO algorithm on the domain $\R$ and a projection technique of \citep{cutkosky2018black}. This gives the state-of-the-art baseline to be introduced shortly. 

\paragraph{Results} Our first result is a lemma connecting the simultaneous achievability of upper bounds on $\loss_T$ and $\reg_T^\unif$ to the achievability of a particular $u$-dependent upper bound on $\reg_T(u)$, possibly of independent interest. It holds with and without the assumption of $g_t\in[-1,1]$. 

\begin{restatable}[Optimal two-point tradeoff via $u$-dependent regret bound]{lemma}{equivalence}\label{lemma:regret_equivalence}
For any $a,b\in\R$ and any OLO algorithm $\mathrm{alg}$, the following two conditions are equivalent: 
\begin{enumerate}
\item $\sup_{\mathrm{adv}}\loss_T(\mathrm{alg},\mathrm{adv})\leq a$ and $\sup_{\mathrm{adv}}\reg_T^\unif(\mathrm{alg},\mathrm{adv})\leq b$;
\item $\sup_{\mathrm{adv},u\in[-1,1]}\reg_T(u,\mathrm{alg},\mathrm{adv})\leq a+(b-a)\abs{u}$. 
\end{enumerate}
\end{restatable}

By Lemma~\ref{lemma:regret_equivalence}, upper-bounding $\gamma(\eps)$ from Eq.\eqref{eq:gamma} can be reduced to finding the lowest $\bar\gamma(\eps)$ such that the regret bound $\reg_T(u)\leq [\eps+(\bar\gamma(\eps)-\eps)\abs{u}]\sqrt{T}$ can be guaranteed up to an $o(\sqrt{T})$ additive error. This fits into the scope of our general results: in particular, Corollary~\ref{corollary:regret} (the regret version of our main theorem) shows that it suffices to pick the $h$ function in Algorithm~\ref{algorithm:main} as the convex conjugate of the targeted $u$-dependent regret bound (viewed as a function of $u$). It can be verified that for any $b\in\R_{>0}$, the convex conjugate of the function 
\begin{equation*}
f(x)=\begin{cases}
a+b\abs{x},&x\in[-1,1],\\
\infty,&\mathrm{else}
\end{cases}
\end{equation*}
is the (vertically shifted) \emph{soft-thresholded absolute value} $f^*(y)=-a+\max\{\abs{y}-b,0\}$. Therefore we are motivated to instantiate Corollary~\ref{corollary:regret} with the soft-thresholded absolute value function, resulting in the following constructive regret upper bound parameterized by a learning rate scalar $\alpha$. 

\begin{restatable}[Regret: soft-thresholded]{corollary}{softProp}\label{corollary:regret_soft}
Assume $\abs{g_t}\leq 1$ for all $t$. If $h(x)=\max\{\abs{x}-\eta^{-1},0\}$ for $\eta=\frac{\alpha}{\sqrt{T}}$ and $\alpha\in\R_{>0}$, then Algorithm~\ref{algorithm:main} with $\rho_t=\sqrt{T-t}$ guarantees
\begin{equation*}
\reg_T(u)\leq \underbrace{\spar{\frac{\abs{u}}{\alpha}+\frac{2}{\alpha}\Phi\rpar{\frac{1}{\alpha}}+2\phi\rpar{\frac{1}{\alpha}}-\frac{2}{\alpha}}}_{\eqdef\gamma_{\mathrm{STh}}(u,\alpha)}\sqrt{T}+O(\log T),\quad\forall u\in[-1,1].
\end{equation*}
\end{restatable}

The obtained Corollary~\ref{corollary:regret_soft} is analogous to Corollary~\ref{corollary:regret_huber} (the case dominating OGD) and Corollary~\ref{corollary:regret_logcosh} (the case dominating MWU). In particular, they all recover the sharp uniform regret upper bound (Corollary~\ref{corollary:regret_absolute}, as well as the results of \citep{kobzar2020a_new,greenstreet2022efficient}) when the learning rate scalar $\alpha$ goes to infinity. Corollary~\ref{corollary:regret_soft} is also additively sharp due to the general regret lower bound, Corollary~\ref{corollary:regret_lower}. 

By reparameterizing the bound using $\eps$ and taking the $T\rightarrow\infty$ limit, we obtain the following exact characterization of $\gamma(\eps)$ from Eq.\eqref{eq:gamma}, realized by a computationally efficient algorithm. 

\begin{restatable}[Characterization of $\gamma(\eps)$]{theorem}{twopoint}
\label{theorem:gamma}
For all constant $\eps\in(0,\sqrt{\frac{2}{\pi}}]$, $\gamma(\eps)$ from Eq.\eqref{eq:gamma} is the unique solution of the equation
\begin{equation*}
\int_{-\infty}^{\eps-\gamma(\eps)}\Phi(x)\diff x=\frac{\eps}{2}.
\end{equation*}
Furthermore, for all $T\in\N_+$, there exists a special case of Algorithm~\ref{algorithm:main} that simultaneously guarantees $\loss_T\leq \eps\sqrt{T}+O(\log T)$ and $\reg^\unif_T\leq \gamma(\eps)\sqrt{T}+O(\log T)$. 
\end{restatable}

In retrospect, the crucial step in the above analysis is using Lemma~\ref{lemma:regret_equivalence} to convert the two-point tradeoff between $\loss_T$ and $\reg_T^\unif$ to a particular ``continuous'' tradeoff on $\reg_T(u)$, such that our general regret upper bound (Corollary~\ref{corollary:regret}) and lower bound (Corollary~\ref{corollary:regret_lower}) can be applied. Alternatively, one may view this conversion from the dual perspective (i.e., total loss upper bounds), which gives another derivation of the right $h$ function to consider by solving an infinite-dimensional linear program. We provide a sketch of this argument in Appendix~\ref{subsection:detail_tradeoff}. 

\begin{figure}[t]
\centering
\includegraphics[width=0.5\linewidth]{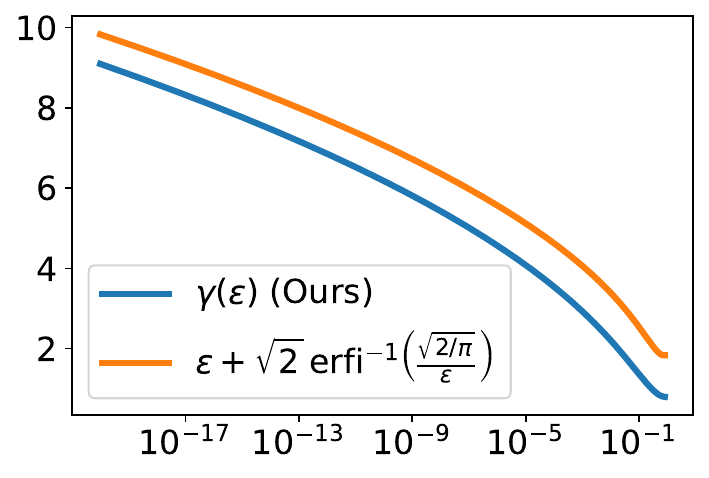}
\caption{Comparison of the prefactors of $\sqrt{T}$ in the $\reg_T^\unif$ upper bound of our algorithm (LHS of Eq.\eqref{eq:two_prefactors}, represented by blue) and the baseline \citep{cutkosky2018black,zhang2022pde} (RHS of Eq.\eqref{eq:two_prefactors}, represented by orange); lower is better. Both are functions of $\eps\in(0,\sqrt{\frac{2}{\pi}}]$ which represents the budget on $\loss_T$. Our result dominates that of the baseline. 
}
\label{figure:two_point}
\end{figure}

\paragraph{Comparison to baseline} Prior to this work, the state-of-the-art algorithm trading off $\loss_T$ and $\reg_T^\unif$ is a combination of \citep{zhang2022pde} and \citep{cutkosky2018black}, concretely introduced in Appendix~\ref{subsection:detail_tradeoff}. With $\erfi:\R\rightarrow\R$ being the \emph{imaginary error function} (i.e., $\erfi(x)=\frac{2}{\sqrt{\pi}}\int_0^x\exp(z^2)\diff z$), such an algorithm runs in $O(1)$ time per round while simultaneously guaranteeing $\loss_T\leq \eps\sqrt{T}$ and
\begin{equation}\label{eq:two_point_baseline}
\reg_T^\unif\leq \eps\sqrt{T}+\sqrt{2T}\cdot \erfi^{-1}\rpar{\frac{\sqrt{2}}{\sqrt{\pi}\eps}}
\end{equation}
against all adversaries satisfying $g_t\in[-1,1]$ for all $t$. 

Comparing our Theorem~\ref{theorem:gamma} to this baseline, we see that up to additive logarithmic terms, both results guarantee the same $\loss_T$ upper bounds, but their $\reg_T^\unif$ bounds differ in the prefactors of $\sqrt{T}$. Demonstrating the advantage of Theorem~\ref{theorem:gamma} amounts to showing
\begin{equation}\label{eq:two_prefactors}
 \gamma(\eps)<\eps+\sqrt{2}\cdot\erfi^{-1}\rpar{\frac{\sqrt{2}}{\sqrt{\pi}\eps}}
\end{equation}
for all $\eps\in(0,\sqrt{\frac{2}{\pi}}]$, which is proved in Lemma~\ref{lemma:gap_soft}. In particular, when $\eps=\sqrt{\frac{2}{\pi}}$, we have $\lhs=\sqrt{\frac{2}{\pi}}$ while $\rhs=\sqrt{\frac{2}{\pi}}+\sqrt{2}\cdot\erfi^{-1}(1)>\lhs$. Also see Figure~\ref{figure:two_point} for visualizations. 

\subsection{Details of the Baseline}\label{subsection:detail_tradeoff}

The baseline we consider starts from a constrained-to-unconstrained reduction first proposed by \citep{cutkosky2018black} and systematically summarized in \citep[Section~9]{orabona2025modern}. To date it has become a key element in the toolbox of online learning. The version for our setting is the following \citep[Remark~9.6]{orabona2025modern}. 

\begin{lemma}[Reduction to unconstrained OLO]\label{lemma:unconstrained_reduction}
Let $\calA$ be an arbitrary algorithm for one-dimensional unconstrained OLO (i.e., $x_t\in\R$ as opposed to $x_t\in[-1,1]$) achieving $\reg_T(u)\leq \psi_T(u)$ for all $u\in\R$, against all adversaries satisfying $g_t\in[-1,1]$ for all $t$. Then, there exists an algorithm (which is a post-processed version of $\A$) satisfying $x_t\in[-1,1]$ for all $t$ while achieving
\begin{equation*}
\reg_T(u)\leq \psi_T(u),\quad\forall u\in[-1,1],
\end{equation*}
against the same adversary class. 
\end{lemma}

Motivated by this lemma, the main objective within the field of parameter-free online learning is to achieve sharp comparator-dependent regret bounds on unconstrained domains. For our setting the state of the art is \citep{zhang2022pde}, building on the techniques of \citep{harvey2023optimal}. Given a potential function defined as
\begin{equation*}
\varphi(t,s)=\eps\sqrt{t}\rpar{\sqrt{\pi}\int_0^{\frac{s}{\sqrt{2t}}}\erfi(z)\diff z-1},
\end{equation*}
the unconstrained algorithm outputs a similar discrete derivative as Cover's algorithm Eq.\eqref{eq:cover},
\begin{equation*}
x_t=-\half\spar{\varphi\rpar{t,\sum_{i=1}^{t-1}g_i+1}-\varphi\rpar{t,\sum_{i=1}^{t-1}g_i-1}}\in\R,\quad\forall t\in[1:T].
\end{equation*}
It guarantees \citep[intermediate result within Theorem~4]{zhang2022pde}
\begin{equation}\label{eq:baseline_unconstrained_regret}
\reg_T(u)\leq \eps\sqrt{T}+\abs{u}\sqrt{2T}\cdot \erfi^{-1}\rpar{\frac{\sqrt{2}\abs{u}}{\sqrt{\pi}\eps}},\quad\forall u\in\R. 
\end{equation}
Setting $\abs{u}=1$ and combining it with Lemma~\ref{lemma:unconstrained_reduction} results in the Eq.\eqref{eq:two_point_baseline} referenced in our work. 

A quantitative strength of this baseline is that Eq.\eqref{eq:baseline_unconstrained_regret} has the optimal constant multiplier $\sqrt{2}$ on the leading order factor $\abs{u}\sqrt{T\log \abs{u}}$, as $\abs{u}$ and $T$ both tend to infinity. Broadly speaking, this is related to the ``additive sharpness'' targeted by our work, while we study optimality in a more detailed scaling regime where $u$ is fixed (on a bounded domain) and only $T\rightarrow\infty$. 

\subsection{Dual Argument of Lemma~\ref{lemma:regret_equivalence} via Linear Programming}\label{subsection:lp_duality}

In this subsection, we sketch the dual argument of Lemma~\ref{lemma:regret_equivalence}, converting the two-point tradeoff between $\loss_T$ and $\reg_T^\unif$ to a particular upper bound function on $\loss_T$ to aim for. Technically, this relies on solving an infinite-dimensional linear program which could be of independent interest. Since our formal results are already derived from Lemma~\ref{lemma:regret_equivalence}, the following sketch will omit many technical subtleties, prioritizing simplicity over mathematical rigor. 

Recall from Theorem~\ref{theorem:main} that for any convex and $1$-Lipschitz function $\psi^*_T:\R\rightarrow\R$ satisfying $\E_{X\sim \calN(0,T)}[\psi_T^*(X)]=0$, there exists a special case of Algorithm~\ref{algorithm:main} guaranteeing
\begin{equation*}
\loss_T\leq -\psi^*_T\rpar{-\sum_{t=1}^Tg_t}+O(\log T)
\end{equation*}
against all adversaries satisfying $g_t\in[-1,1]$ for all $t$. This is additively sharp due to Theorem~\ref{theorem:loss_lower}. We are now interested in finding the optimal tradeoff between $\loss_T$ and $\reg^\unif_T$ induced by such $\psi^*_T$ functions. 

For the simplicity of this sketch, we restrict the search space to $\psi^*_T$ functions that are also even. Given this restriction, notice that the above $\loss_T$ upper bound leads to the adversary-independent upper bound
\begin{equation*}
\loss_T\leq \sup_{s\in\R}\spar{-\psi^*_T\rpar{s}}+O(\log T)=-\psi^*_T\rpar{0}+O(\log T),
\end{equation*}
as well as the uniform regret bound (via the loss-regret duality, Lemma~\ref{lemma:duality})
\begin{align*}
\reg^\unif_T&\leq \max_{u\in\{-1,1\}}(\psi^*_T)^*(u)+O(\log T)\\
&=\sup_{s\in\R} [s-\psi^*_T(s)]+O(\log T)\tag{$\psi^*_T$ is even}\\
&=\lim_{s\rightarrow\infty}[s-\psi^*_T(s)]+O(\log T).\tag{$\psi^*_T$ is $1$-Lipschitz}
\end{align*}
Omitting the additive $O(\log T)$ terms from the RHS, the task of minimizing $\reg^\unif_T$ under the constraint $\loss_T\leq \eps\sqrt{T}$ can be modeled as the following optimization problem,
\begin{align*}
\textrm{Problem One}\leftarrow
\left\{~
\begin{aligned}
\min_{f:\R\rightarrow\R} \quad & \lim_{x\rightarrow\infty} [x-f(x)]\\
    \textrm{s.t.} \quad &\E_{X\sim \calN(0,T)}[f(X)]=0,\\
    &-f(0)=\eps\sqrt{T},\\
    & f~\textrm{is even, convex and 1-Lipschitz.}
\end{aligned}
\right.
\end{align*}

The idea is that the optimal value of Problem One normalized by $\frac{1}{\sqrt{T}}$ gives $\gamma(\eps)$. While additional effort is required to formally establish their equivalence, here we do not treat such subtleties. The minimizing argument $f$ is the total loss upper bound function to aim for, thus also the $h$ function to apply in Algorithm~\ref{algorithm:main}. The subsequent sketch is divided into four steps. 

\paragraph{Step 1: reformulation using derivative} To simplify Problem One, notice that for all feasible function $f$, the derivative $f'$ exists almost everywhere. Since
\begin{equation*}
f(x)=f(0)+\int_0^xf'(s)\diff s,\quad\forall x\in\R,
\end{equation*}
we have
\begin{equation*}
\lim_{x\rightarrow\infty} [x-f(x)]=-f(0)+\int_0^\infty [1-f'(x)]\diff x,
\end{equation*}
\begin{align*}
\E_{X\sim\calN(0,T)}[f(X)]&=2\int_0^\infty f(\sqrt{T}z)\phi(z)\diff z\\
&=2\int_0^\infty \rpar{f(0)+\int_0^{\sqrt{T}z}f'(s)\diff s}\phi(z)\diff z\\
&=f(0)+2\int_0^\infty\rpar{\int_{\frac{s}{\sqrt{T}}}^\infty\phi(z)\diff z}f'(s)\diff s\tag{swap integrals}\\
&=f(0)+2\int_0^\infty\spar{1-\Phi\rpar{\frac{s}{\sqrt{T}}}}f'(s)\diff s.
\end{align*}
Combining these properties, we can change the argument of Problem One from $f$ to its derivative $f'$, thereby converting the problem to the following infinite dimensional linear program that shares the same optimal value. 
\begin{align*}
\textrm{Problem Two}\leftarrow\left\{~
\begin{aligned}
\min_{g:\R_{\geq 0}\rightarrow[0,1]} \quad & \eps\sqrt{T}+\int_0^\infty [1-g(x)]\diff x\\
    \textrm{s.t.} \quad & - \eps\sqrt{T}+2\int_0^\infty\spar{1-\Phi\rpar{\frac{x}{\sqrt{T}}}}g(x)\diff x=0.
\end{aligned}
\right.
\end{align*}
Notably, compared to Problem One, Problem Two drops the monotonicity requirement on the function $g$. This is a tight relaxation which can be verified. 

\paragraph{Step 2: the dual problem} The second step is to apply the standard duality analysis from convex optimization. With the multiplier $\lambda\in\R$, the Lagrangian of Problem Two is
\begin{align*}
L(g,\lambda)&\defeq \eps\sqrt{T}+\int_0^\infty [1-g(x)]\diff x+\lambda\rpar{2\int_0^\infty\spar{1-\Phi\rpar{\frac{x}{\sqrt{T}}}}g(x)\diff x-\eps\sqrt{T}}\\
&=(1-\lambda)\eps\sqrt{T}+\int_0^\infty \spar{1+\rpar{2\lambda-2\lambda\Phi\rpar{\frac{x}{\sqrt{T}}}-1}g(x)}\diff x.
\end{align*}
Minimizing it with respect to $g:\R\rightarrow[0,1]$ is equivalent to assigning $g(x)$ pointwise to $0$ or $1$, giving the Lagrange dual function
\begin{equation*}
q(\lambda)\defeq \min_{g:\R_{\geq 0}\rightarrow[0,1]}L(g,\lambda)=(1-\lambda)\eps\sqrt{T}+\int_0^\infty \min\left\{1,2\lambda\spar{1-\Phi\rpar{\frac{x}{\sqrt{T}}}}\right\}\diff x.
\end{equation*}
Accordingly, the dual of Problem Two is maximizing $q(\lambda)$ over $\lambda\in\R$. 

\paragraph{Step 3: solving the dual problem} Regarding this dual problem, it can be verified that the maximizing argument satisfies $\lambda\geq 1$. For any such $\lambda$, define the notation
\begin{equation*}
x_{\lambda}\defeq \sqrt{T}\cdot\Phi^{-1}\rpar{1-\frac{1}{2\lambda}}\geq 0,
\end{equation*}
such that
\begin{equation*}
q(\lambda)=(1-\lambda)\eps\sqrt{T}+x_\lambda+2\lambda\int_{x_\lambda}^\infty \spar{1-\Phi\rpar{\frac{x}{\sqrt{T}}}}\diff x.
\end{equation*}
Differentiating it gives
\begin{align*}
q'(\lambda)&=-\eps\sqrt{T}+\frac{\partial x_{\lambda}}{\partial\lambda}+2\int_{x_\lambda}^\infty \spar{1-\Phi\rpar{\frac{x}{\sqrt{T}}}}\diff x-2\lambda\spar{1-\Phi\rpar{\frac{x_\lambda}{\sqrt{T}}}}\frac{\partial x_{\lambda}}{\partial\lambda}\\
&=-\eps\sqrt{T}+2\int_{x_\lambda}^\infty \spar{1-\Phi\rpar{\frac{x}{\sqrt{T}}}}\diff x.
\end{align*}
The maximizing argument $\lambda^*$ satisfies $q(\lambda^*)=0$, which is equivalent to the solution of
\begin{equation*}
\int_{-\infty}^{-\Phi^{-1}\rpar{1-\frac{1}{2\lambda^*}}}\Phi(x)\diff x=\frac{\eps}{2}.
\end{equation*}

\paragraph{Step 4: final result} Summarizing the above and observing that strong duality holds, we have
\begin{equation*}
\gamma(\eps)=\frac{q(\lambda^*)}{\sqrt{T}}=\eps+\Phi^{-1}\rpar{1-\frac{1}{2\lambda^*}},
\end{equation*}
therefore $\int_{-\infty}^{\eps-\gamma(\eps)}\Phi(x)\diff x=\frac{\eps}{2}$, recovering the result from Theorem~\ref{theorem:gamma}. Furthermore, the minimizing argument $f$ of Problem One is given by the soft-thresholded absolute value, 
\begin{equation*}
f(x)=-\eps\sqrt{T}+\max\left\{\abs{x}-\sqrt{T}\cdot\Phi^{-1}\rpar{1-\frac{1}{2\lambda^*}},0\right\},\quad\forall x\in\R.
\end{equation*}
This is the $h$ function used to instantiate our algorithm in Corollary~\ref{corollary:regret_soft}. 

\section{Extension: OLO with Noisy Feedback}\label{section:stochastic}

This section extends the algorithmic results so far to the setting with stochastic adversaries, also known as OLO with noisy feedback. By doing this, we further develop the parallel relation between performance guarantees in OLO and Wasserstein martingale CLTs \citep{rollin2018quantitative}. Motivations are two-fold.
\begin{itemize}
\item The intuition underlying the proof of our main theorem (Theorem~\ref{theorem:main}) is to view the adversary's cumulative decision $\sum_{t=1}^Tg_t$ as a stochastic process indexed by $T$. After a suitable notion of bias is properly taken care of, this can be essentially analyzed as a martingale whose normal approximation properties follow from Stein's method. While so far we have only considered deterministic $\sum_{t=1}^Tg_t$, it is natural to expect that the same idea still works when $\sum_{t=1}^Tg_t$ is by default a stochastic process. This section performs this extension. 
\item In OLO, achievable performance guarantees are governed by the complexity of the targeted adversary class. While typically the adversaries are constrained in a deterministic manner (e.g., by assuming $g_t\in[-1,1]$), an alternative setting is to impose soft, distributional constraints such that $\abs{g_t}$ is allowed to be large (albeit with low probability). There have been plenty of works on this extension \citep{jun2019parameter,van2019user,zhang2022parameter,bhatt2025prediction,liu2025online}, but to our knowledge, achieving better-than-big-O optimality remains open. 
\end{itemize}

After introducing the setting, we present the extensions of our main results in Appendix~\ref{subsection:stochastic_main}. The relation to R{\"o}llin's martingale CLT is discussed in Appendix~\ref{subsection:clt}. 

\paragraph{Setting} With a time horizon $T\in\N_+$, let $\Omega=\R^T$ and $\calF=\calB(\R^T)$ represent the sequence space and its Borel $\sigma$-algebra. Let $G_t(\omega)=\omega_t, \omega\in\Omega,t\in[1:T]$ be the canonical coordinate process (intuitively, the \emph{loss process} in OLO), and let $\calF_t=\sigma(G_1,\ldots,G_t),t\in[0,T]$ be the canonical filtration. Consider the following interaction between an OLO algorithm and an adversary, analogous to the deterministic setting in Section~\ref{section:introduction}.
\begin{itemize}
\item The algorithm is defined by a sequence of Borel functions $(a_t)_{t\in[1:T]}$ where each $a_t:\R^{t-1}\rightarrow[-1,1]$ can depend on $T$. This defines a \emph{decision process} $(X_t)_{t\in[1:T]}$ where
\begin{equation*}
X_t(\omega)=a_t(G_1(\omega),\ldots,G_{t-1}(\omega)),
\end{equation*}
which is predictable with respect to $(\calF_{t})_{t\in[0:T]}$.
\item The adversary is defined by a probability measure $\P$ on $(\Omega,\calF)$, which can arbitrarily depend on $T$ and the algorithm $(a_t)_{t\in[1:T]}$. This specifies the law of the loss process $(G_t)_{t\in[1:T]}$, and by fixing the algorithm, the law of the decision process $(X_t)_{t\in[1:T]}$ as well. 
\end{itemize}

On the filtered probability space $(\Omega,\calF,(\calF_{t})_{t\in[0:T]},\P)$, the performance of the algorithm is measured by the total loss
\begin{equation*}
\loss_T=\sum_{t=1}^TG_tX_t
\end{equation*}
which is a random variable, and the goal of the algorithm is to minimize its expectation in the same sense as in Section~\ref{section:introduction}. We use the plain $\E$ to denote the expectation with respect to $\P$. This is not to be confused with $\E_Z$ in the rest of our analysis, which takes the expectation with respect to the external randomness in $Z$. Picking $\P$ as a Dirac measure essentially recovers the deterministic setting. 

For notational convenience, we also write $x_t=X_t(\omega)$ and $g_t=G_t(\omega)$ such that all the pathwise notations are consistent with the deterministic setting.

\paragraph{OLO with noise} The above is equivalent to the classical problem of OLO with noisy feedback: after a standard deterministic adversary picks $g_t$, the algorithm observes $\hat g_t=g_t+\sigma_t$ where the noise $\sigma_t$ is conditionally zero mean, and the performance of the algorithm is evaluated by the expectation of its total loss $\sum_{t=1}^T\hat g_tx_t$. We note that our analysis does not place explicit assumptions on the tail of the noise $\sigma_t$ (e.g., sub-Gaussianity) which are required by typical martingale concentration inequalities. Instead, assumptions will be imposed on the conditional moments of $\sigma_t$. This will be made clear in Corollary~\ref{corollary:case_bounded_moment}. 

\subsection{Main Result}\label{subsection:stochastic_main}

\paragraph{Algorithm} Consider applying Algorithm~\ref{algorithm:main} to every sample path of the above stochastic setting. To prepare for its population-level guarantees, we now reformulate the intermediate quantities of Algorithm~\ref{algorithm:main} as stochastic processes. 
\begin{itemize}
\item On each sample path, Algorithm~\ref{algorithm:main} requires picking a parameter $\rho_t$ based on $g_1,\ldots,g_{t-1}$. On the population level, with $\rho_t$ represented by some $P_t(\omega)$, this amounts to the algorithm picking a predictable process $(P_t)_{t\in[0:T]}$ with respect to $(\calF_{t})_{t\in[0:T]}$, such that ($i$) for all $\omega\in\Omega$ and $t\in[1:T]$, $P_t(\omega)\in[0,P_{t-1}(\omega)]$; and ($ii$) for all $\omega\in\Omega$ and $t<T$, $P_t(\omega)>0$. We note that $P_0$ is simply a positive real number. 

\item Accordingly, based on the pathwise quantity $c_t=\rho_{t-1}^2-\rho^2_{t}$, we define the process $(C_t)_{t\in[1:T]}$ by
$C_t=P^2_{t-1}-P^2_t$, which is also predictable with respect to $(\calF_{t})_{t\in[0:T]}$.
\item Finally, define the process $(S_t)_{t\in[1:T]}$ by $S_t=\sum_{t=1}^TG_t$.
\end{itemize}
Using these notations, applying Algorithm~\ref{algorithm:main} to the stochastic setting is equivalent to the update rule
\begin{equation}\label{eq:output_stochastic}
X_t=\E_{Z}\spar{f_{S_{t-1},P_{t-1},h}(S_{t-1}+P_tZ)}. 
\end{equation}

\paragraph{Total loss bound} A key strength of our main result (Theorem~\ref{theorem:main}) is that the stated total loss bound holds against arbitrary adversaries, including those that are unconstrained beforehand. Therefore taking the expectation of Theorem~\ref{theorem:main} (more precisely, a partial result in its proof) leads to the following theorem.

\begin{theorem}[Expected total loss bound]\label{theorem:stochastic}
In the stochastic setting, against all adversaries, Algorithm~\ref{algorithm:main} (i.e., Eq.\eqref{eq:output_stochastic}) given any convex and $1$-Lipschitz function $h$ guarantees
\begin{multline*}
\E[\loss_T]
\leq -\E\spar{\E_{Z} \spar{h\rpar{\sum_{t=1}^TG_t+P_T Z}}}+\E_{Z}[h(P_0Z)]\\
+\E\sum_{t=1}^T\rpar{\sqrt{\frac{2}{\pi}}\frac{\max\{\E[G_t^2|\calF_{t-1}]-C_t,0\}}{P_{t-1}}+\frac{2C_t\abs{G_t}+\abs{G_t^3}}{P_{t-1}^2}}.
\end{multline*}
\end{theorem}

\begin{proof}[Proof of Theorem~\ref{theorem:stochastic}]
We start from the following partial result in the proof of Theorem~\ref{theorem:main} which holds pathwise on the filtered probability space $(\Omega,\calF,(\calF_{t})_{t\in[0:T]},\P)$,
\begin{multline*}
\E_{Z}[h(S_t+P_tZ)]-\E_{Z}[h(S_{t-1}+P_{t-1}Z)]
\\
\leq-G_tX_{t}+(C_t-G_t^2)\E_{Z}\spar{f'_{S_{t-1},P_{t-1},h}(S_{t-1}+P_tZ)}+\frac{2C_t\abs{G_t}+\abs{G_t^3}}{P_{t-1}^2}.
\end{multline*}
Taking the expectation conditioned on $\calF_{t-1}$ before applying Lemma~\ref{lemma:stein_magnitude} (Stein factors) and Lemma~\ref{lemma:monotonicity} ($f'_{s_{t-1},\rho_{t-1},h}$ is non-positive) gives
\begin{multline*}
\E\spar{\E_{Z}[h(S_t+P_tZ)]-\E_{Z}[h(S_{t-1}+P_{t-1}Z)]\Big|\calF_{t-1}}\\
\leq-\E[G_t|\calF_{t-1}]X_t+\sqrt{\frac{2}{\pi}}\frac{\max\{\E[G_t^2|\calF_{t-1}]-C_t,0\}}{P_{t-1}}+\frac{2C_t\E[\abs{G_t}|\calF_{t-1}]+\E[\abs{G_t^3}|\calF_{t-1}]}{P_{t-1}^2}.
\end{multline*}
Plugging it back into the telescopic sum Eq.\eqref{eq:telescopic} from the proof of Theorem~\ref{theorem:main} and taking the expectation completes the proof.
\end{proof}

Despite the generality of Theorem~\ref{theorem:stochastic}, the most interpretable case is when the second moment of $G_t$ conditioned on $\calF_{t-1}$ is bounded by a known constant, e.g., $\E[G_t^2|\calF_{t-1}]\leq 1$. In this case we obtain the following simplification by setting $P_t$ and $C_t$ deterministically. 

\begin{corollary}[Bounded second moment]\label{corollary:case_bounded_moment}
In the stochastic setting, if the adversary satisfies the moment condition $\E[G_t^2|\calF_{t-1}]\leq 1$ for all $t\in[1:T]$, then by setting $P_t=\sqrt{T-t}$ deterministically, Algorithm~\ref{algorithm:main} (i.e., Eq.\eqref{eq:output_stochastic}) given any convex and $1$-Lipschitz function $h$ guarantees
\begin{equation*}
\E[\loss_T]
\leq -\E\spar{h\rpar{\sum_{t=1}^TG_t}}+\E_{Z}[h(\sqrt{T}Z)]
+\sum_{t=1}^T\frac{2\E[\abs{G_t}]+\E[\abs{G_t^3}]}{T-t+1}.
\end{equation*}
\end{corollary}

The value of this total loss bound is its dependence on the first three moments of $G_t$ as opposed to more specific tail conditions. 
If $G_t$ for all $t\in[1:T]$ is supported on the interval $[-1,1]$, then the sum on the RHS is $O(\log T)$, meaning that Corollary~\ref{corollary:case_bounded_moment} essentially recovers its deterministic counterpart (Theorem~\ref{theorem:main}) against $g_t\in[-1,1]$ adversaries which is additively sharp. We also note that just like Theorem~\ref{theorem:main}, Corollary~\ref{corollary:case_bounded_moment} holds for all convex and $1$-Lipschitz $h$ functions, therefore results from other sections of this work can be applied in an orthogonal manner.

\paragraph{Without convexity} Next we provide a variant of Theorem~\ref{theorem:stochastic} to be used shortly. Notice that when proving Theorem~\ref{theorem:stochastic}, the convexity of the $h$ function is only used through Lemma~\ref{lemma:monotonicity}. Removing this step gives the following analogous result without convexity, whose proof is nearly identical to that of Theorem~\ref{theorem:stochastic}. Implications will be further discussed in Appendix~\ref{section:nonconvex}. 

\begin{theorem}[Variant of Theorem~\ref{theorem:stochastic} without convexity]\label{theorem:stochastic_nonconvex}
In the stochastic setting, against all adversaries, Algorithm~\ref{algorithm:main} (i.e., Eq.\eqref{eq:output_stochastic}) given any $1$-Lipschitz function $h$ guarantees
\begin{multline*}
\E[\loss_T]
\leq -\E\spar{\E_{Z} \spar{h\rpar{\sum_{t=1}^TG_t+P_T Z}}}+\E_{Z}[h(P_0Z)]\\
+\E\sum_{t=1}^T\rpar{\sqrt{\frac{2}{\pi}}\frac{\abs{\E[G_t^2|\calF_{t-1}]-C_t}}{P_{t-1}}+\frac{2C_t\abs{G_t}+\abs{G_t^3}}{P_{t-1}^2}}.
\end{multline*}
\end{theorem}

\subsection{From Total Loss Bound to Martingale CLT}\label{subsection:clt}

The construction of our OLO algorithm is based on the remarkably clean proof of a quantitative Wasserstein martingale CLT due to \citep{rollin2018quantitative}. We now further elucidate their connection from the opposite direction, more specifically showing that R{\"o}llin's martingale CLT can be proved via the above Theorem~\ref{theorem:stochastic_nonconvex}. In this way, we justify viewing the obtained total loss upper bounds as algorithmic analogues of the martingale CLT. 

Concretely, we first restate the result of R{\"o}llin \citep[Theorem~2.1]{rollin2018quantitative}. For two random variables $X$ and $Y$ taking values in $\R$, define their Wasserstein distance as
\begin{equation*}
d_{\mathrm{W}}(X,Y)\defeq \sup_{h\in\mathrm{Lip}_1(\R)}\abs{\E[h(X)]-\E[h(Y)]},
\end{equation*}
where $\mathrm{Lip}_1(\R)$ denotes the collection of all $1$-Lipschitz functions from $\R$ to $\R$. 

\begin{theorem}[R{\"o}llin's martingale CLT]\label{theorem:clt}
For a martingale $(L_t)_{t\in[1:T]}$ adapted to the filtration $(\calF_t)_{t\in[0:T]}$, define the notations $V_T=\sum_{t=1}^T\E[L^2_t|\calF_{t-1}]$ and $v_T=\sum_{t=1}^T\E[L^2_t]$. Assume that $V_T=v_T$ almost surely. Then, for any $a\geq 0$,
\begin{equation*}
d_{\mathrm{W}}\rpar{\frac{1}{\sqrt{v_T}}\sum_{t=1}^TL_t,Z}\leq \frac{1}{\sqrt{v_T}}\spar{3\sum_{t=1}^T\E\frac{\abs{L_t}^3}{v_T-\sum_{i=1}^{t-1}\E[L_i^2|\calF_{i-1}]+a^2}+2a}.
\end{equation*}
\end{theorem}

Next, we give a concise rephrasing of R{\"o}llin's original proof starting from our Theorem~\ref{theorem:stochastic_nonconvex}. 

\begin{proof}[Proof of Theorem~\ref{theorem:clt}]
Consider applying Theorem~\ref{theorem:stochastic_nonconvex} with the loss process $(G_t)_{t\in[1:T]}$ being the martingale $(L_t)_{t\in[1:T]}$ from Theorem~\ref{theorem:clt}. The algorithm from Theorem~\ref{theorem:stochastic_nonconvex} requires choosing the process $(P_t)_{t\in[1:T]}$, and here we consider equipping it with oracle knowledge: let $C_t=\E[L^2_t|\calF_{t-1}]$, $P_0=\sqrt{\sum_{t=1}^T\E[L^2_t]+a^2}$, and thus
\begin{align*}
&P_t=\sqrt{P_0^2-\sum_{t=1}^tC_t}=\sqrt{v_T-\sum_{i=1}^t\E[L^2_i|\calF_{i-1}]+a^2},\\
&P_T=\sqrt{P_0^2-\sum_{t=1}^TC_t}=\sqrt{v_T-V_T+a^2}=a,\quad a.s.
\end{align*}
The last equality is due to the assumption that $V_T=v_T$ almost surely. We note that such choices of $(C_t)_{t\in[1:T]}$ and $(P_t)_{t\in[1:T]}$ satisfy their respective measurability requirements. Then, Theorem~\ref{theorem:stochastic_nonconvex} yields
\begin{equation*}
\E\spar{\sum_{t=1}^TL_tX_t}\leq -\E\spar{\E_{Z} \spar{h\rpar{\sum_{t=1}^TL_t+a Z}}}+\E_{Z}[h(\sqrt{v_T+a^2}Z)]
+\E\sum_{t=1}^T\frac{2\E[L^2_t|\calF_{t-1}]\abs{L_t}+\abs{L_t^3}}{v_T-\sum_{i=1}^{t-1}\E[L^2_i|\calF_{i-1}]+a^2},
\end{equation*}
where by the Lipschitzness of $h$,
\begin{equation*}
-\E\spar{\E_{Z} \spar{h\rpar{\sum_{t=1}^TL_t+a Z}}}\leq -\E\spar{h\rpar{\sum_{t=1}^TL_t}}+\E_Z[\abs{aZ}]\leq -\E\spar{h\rpar{\sum_{t=1}^TL_t}}+a,
\end{equation*}
\begin{equation*}
\E_{Z}[h(\sqrt{v_T+a^2}Z)]\leq \E_{Z}[h(\sqrt{v_T}Z)]+a.
\end{equation*}
Furthermore,
\begin{equation*}
\E\sum_{t=1}^T\frac{2\E[L^2_t|\calF_{t-1}]\abs{L_t}+\abs{L_t^3}}{v_T-\sum_{i=1}^{t-1}\E[L^2_i|\calF_{i-1}]+a^2}\leq \E\sum_{t=1}^T\frac{3\abs{L_t^3}}{v_T-\sum_{i=1}^{t-1}\E[L^2_i|\calF_{i-1}]+a^2}.
\end{equation*}

Since $(X_t)_{t\in[1:T]}$ is predictable with respect to the natural filtration $(\calF_t)_{t\in[1:T]}$ of $(L_t)_{t\in[1:T]}$ and the latter is a martingale, we have $\E[\sum_{t=1}^TL_tX_t]=0$. Therefore combining the above, for any $1$-Lipschitz $h$ function we have
\begin{equation*}
\E\spar{h\rpar{\sum_{t=1}^TL_t}}-\E_{Z}[h(\sqrt{v_T}Z)]
\leq 3\E\sum_{t=1}^T\frac{\abs{L_t^3}}{v_T-\sum_{i=1}^{t-1}\E[L^2_i|\calF_{i-1}]+a^2}+2a.
\end{equation*}
Notice that the inequality still holds with $h\leftarrow-h$. Finally scaling all $L_t$ and $a$ by $\frac{1}{\sqrt{v_T}}$ completes the proof.
\end{proof}

We remark that the major technical elements of this proof are already presented in \citep{rollin2018quantitative}, with some of them being classical techniques of Stein's method (Appendix~\ref{section:background_stein}). Nonetheless, we find it conceptually very interesting to rephrase the proof through performance guarantees in OLO, indicating that many of their technical difficulties are in common. A notable observation is that the instrumental OLO algorithm constructed within the proof is allowed to have prior knowledge on the quadratic variation of the adversary, whereas in the actual OLO problem the algorithm does not have such prior knowledge therefore the convexity of the $h$ function becomes crucial. Formally connecting this observation to game-theoretic probability \citep{shafer2019game} would be a valuable direction for future works.  

\section{Discussion of Nonconvex Tradeoffs}\label{section:nonconvex}

Deviating from the rest of this work, this section discusses nonconvex performance tradeoffs in OLO, i.e., achieving
\begin{equation}\label{eq:loss_bound_nonconvex}
\loss_T\leq -\mathrm{Bound}\rpar{-\sum_{t=1}^Tg_t}
\end{equation}
analogous to Eq.\eqref{eq:loss_bound} but the function $\mathrm{Bound}:\R\rightarrow\R$ here is nonconvex. Since the convex conjugate of any function is automatically convex, the function $\mathrm{Bound}$ cannot be expressed as a convex conjugate $\psi^*_T$, therefore total loss bounds of this type cannot be recovered by regret bounds through the loss-regret duality. For such a nonconvex tradeoff regime, existing results are scarce. 

First, we note that Cover's characterization of performance tradeoffs in OLO (Theorem~\ref{theorem:cover}) also works in this nonconvex regime since the adversary is assumed to be Boolean. Concretely, the following generalization of Theorem~\ref{theorem:cover} is true: for all $1$-Lipschitz function $\mathrm{Bound}:\R\rightarrow(-\infty,\infty]$, there exists an algorithm achieving the total loss bound Eq.\eqref{eq:loss_bound_nonconvex} against Boolean adversaries if and only if
\begin{equation*}
\E_{X\sim \rs(T)}[\mathrm{Bound}(X)]\leq 0.
\end{equation*}
The associated algorithm is still given by Eq.\eqref{eq:cover}, and our goal is to give a computationally efficient improvement. 

By inspection, the proof of our Theorem~\ref{theorem:main} only uses the convexity of the $h$ function through Lemma~\ref{lemma:monotonicity}. Therefore when the $h$ function given to Algorithm~\ref{algorithm:main} is $1$-Lipschitz but not necessarily convex, the following variant of Theorem~\ref{theorem:main} is immediate. Also see Theorem~\ref{theorem:stochastic_nonconvex} for the analogue of this result in the stochastic setting of OLO. 

\begin{theorem}[Variant of Theorem~\ref{theorem:main} without convexity]\label{theorem:main_nonconvex}
Against all adversaries (potentially unbounded a priori), Algorithm~\ref{algorithm:main} given any $1$-Lipschitz function $h$ guarantees
\begin{multline*}
\loss_T\\
\leq -\E_{Z} \spar{h\rpar{\sum_{t=1}^Tg_t+\rho_T Z}}+\E_{Z}[h(\rho_0Z)]+\sum_{t=1}^T\rpar{\sqrt{\frac{2}{\pi}}\frac{\abs{g_t^2-c_t}}{\rho_{t-1}}+\frac{2c_t\abs{g_t}+\abs{g_t^3}}{\rho_{t-1}^2}}.
\end{multline*}
\end{theorem}

Similar to the idea of Theorem~\ref{theorem:main}, the first two terms on the RHS essentially match Cover's achievability condition therefore constitute the ideal bound to aim for, whereas the remaining sum should be understood as the discretization error. To maximize the interpretability, consider again the case where $\abs{g_t}\leq 1$ for all $t$ and accordingly Algorithm~\ref{algorithm:main} sets $\rho_t=\sqrt{T-t}$ and $c_t=1$. We have
\begin{equation*}
\sum_{t=1}^T\frac{2c_t\abs{g_t}+\abs{g_t^3}}{\rho_{t-1}^2}=O(\log T),
\end{equation*}
but the other part of the discretization error can only be generally bounded by
\begin{equation*}
\sum_{t=1}^T\sqrt{\frac{2}{\pi}}\frac{\abs{g_t^2-c_t}}{\rho_{t-1}}=\max_{t\in[1:T]}\abs{g_t^2-c_t}\cdot O(\sqrt{T}).
\end{equation*}
While this is still a meaningful result, unless $c_t=g_t^2$ for all $t$ we can no longer argue about the additive sharpness of the obtained total loss bound since the $O(\sqrt{T})$ discretization error is often significant.

To summarize, in this nonconvex tradeoff regime: 
\begin{itemize}
\item If the adversary is Boolean, then by setting $\rho_t=\sqrt{T-t}$ thus $g_t^2=c_t$ for all $t$, we obtain a computationally efficient improvement of Cover's algorithm, and the associated total loss bound is additively sharp (akin to the convex tradeoff regime, Theorem~\ref{theorem:main}). 
\item If the adversary satisfies $\abs{g_t}\leq 1$ but is not Boolean, then Cover's total loss bound does not apply, whereas our bound is meaningful but not additively sharp. This is due to an inflated discretization error term related to the sequence of misspecification error $\abs{g_t^2-c_t}$. 
\end{itemize}

\section{Omitted Proofs}\label{section:omitted_proofs}

This section presents the omitted proofs thus far. Appendix~\ref{subsection:reformulation_general} proves the complete version of Lemma~\ref{lemma:equivalent_short}, giving a more direct representation of Algorithm~\ref{algorithm:main}'s output. Appendix~\ref{subsection:omitted_proofs_general} proves Theorem~\ref{theorem:loss_lower} (our general lower bound on $\loss_T$), and then uses the loss-regret duality to convert both Theorem~\ref{theorem:main} and Theorem~\ref{theorem:loss_lower} to matching upper and lower bounds on $\reg_T(u)$. Appendix~\ref{subsection:omitted_proofs_cases} proves the instantiations of these results on iconic $h$ functions. Appendix~\ref{subsection:omitted_two_point} presents the omitted proofs for Appendix~\ref{section:soft_thresholded}. Appendix~\ref{subsection:omitted_technical} proves a number of technical lemmas. Appendix~\ref{subsection:omitted_existing} summarizes several classical lemmas we use, whose proofs are omitted. 

\subsection{Rewriting the Algorithm's Output}\label{subsection:reformulation_general}

The following lemma is the complete version of Lemma~\ref{lemma:equivalent_short} in Section~\ref{section:algorithm}.  

\begin{lemma}[Rewriting the output]\label{lemma:equivalent}
The output Eq.\eqref{eq:algorithm} of Algorithm~\ref{algorithm:main} is equivalent to
\begin{align*}
x_{t}&=-\frac{1}{2}\int_0^{1}\frac{1}{\sqrt{\tau}}\E_Z\spar{h'\rpar{s_{t-1}+\sqrt{\rho^2_{t-1}-\tau c_t}Z}}\diff\tau\\
&=-\E_{Z\sim\calN(0,1);\tau\sim\mathrm{Beta}(\half,1)}\spar{h'\rpar{s_{t-1}+\sqrt{\rho^2_{t-1}-\tau c_t}Z}},
\end{align*}
as well as
\begin{align*}
x_{t}&=-\int_{0}^{\infty}e^{-\tau}\E_Z\spar{h'\rpar{s_{t-1}+\sqrt{\rho^2_{t-1}-e^{-2\tau}c_t}Z}}\diff\tau\\
&=-\E_{Z\sim\calN(0,1);\tau\sim\mathrm{Exp}(1)}\spar{h'\rpar{s_{t-1}+\sqrt{\rho^2_{t-1}-e^{-2\tau} c_t}Z}}.
\end{align*}
\end{lemma}

An immediate remark is that the integrals are well-posed: due to the Lipschitzness of $h'$, $h'$ exists almost everywhere, and whenever $h'$ exists it is bounded. Therefore the expectation of $h'$ exists on distributions with density. 

\begin{proof}[Proof of Lemma~\ref{lemma:equivalent}]
Observe that the two equations in the lemma are equivalent by a change of variable. To connect them to Eq.\eqref{eq:algorithm}, the proof is divided into two parts to avoid the technical nuances of taking the $\rho_t\rightarrow 0$ limit. The first part tackles the degenerate case of $\rho_{t}=0$ by brute force, which is only possible at $t=T$ by the requirement of Algorithm~\ref{algorithm:main}. The second part tackles the nondegenerate case of $\rho_t>0$ where a more structured continuous argument is applicable. 

\paragraph{Part 1} If $\rho_{t}=0$, i.e., $t=T$ and $\rho_{t-1}^2=c_t$, then Eq.\eqref{eq:algorithm} becomes $x_{t}=f_{s_{t-1},\rho_{t-1},h}(s_{t-1})$. Using the integral representation of the function $f_{\mu,\sigma,h}$ through the OU semigroup (Eq.\eqref{eq:ou_rep} in Appendix~\ref{section:background_stein}), we have
\begin{equation*}
f_{\mu,\sigma,h}(\mu)=-\int_0^\infty e^{-\tau}\E_Z\spar{h'\rpar{\mu+\sigma\sqrt{1-e^{-2\tau}}Z}}\diff \tau.
\end{equation*}
Plugging in $\mu\leftarrow s_{t-1}$ and $\sigma\leftarrow\rho_{t-1}$ results in the equivalence of Eq.\eqref{eq:algorithm} and the second equation of the lemma. 

\paragraph{Part 2} Next consider $\rho_t\neq 0$, i.e., $\rho^2_{t-1}>c_t\geq 0$. Define the notations $Y\sim\calN(\mu,\sigma^2-c)$ for general $\sigma^2>c\geq 0$, and $u(c)\defeq\E_Y[f_{\mu,\sigma,h}(Y)]=\E_Z[f_{\mu,\sigma,h}(\mu+\sqrt{\sigma^2-c}Z)]$. The goal of Part 2 is to prove $u(0)=-\E_Z[h'(\mu+\sigma Z)]$ and
\begin{equation}\label{eq:u_of_c}
u(c)=-\frac{1}{2\sqrt{c}}\int_0^c\frac{1}{\sqrt{\tau}}\E_Z\spar{h'\rpar{\mu+\sqrt{\sigma^2-\tau}Z}}\diff\tau,\quad \forall c\in(0,\sigma^2). 
\end{equation}
After that, by plugging in $\mu\leftarrow s_{t-1}$, $\sigma\leftarrow\rho_{t-1}$ and $c\leftarrow c_t$, the LHS $u(c)$ recovers the definition of $x_t$ in Eq.\eqref{eq:algorithm} while the RHS recovers the first equation of the objective lemma. Starting from this point, the subscripts of $f_{\mu,\sigma,h}$ will be dropped for conciseness. 

\textit{Step 1.} Our first step is to differentiate the Stein equation Eq.\eqref{eq:stein} with respect to $x$, where all derivatives exist almost everywhere. 
\begin{equation*}
\sigma^2f''(x)-(x-\mu)f'(x)-f(x)=h'(x),\quad a.e. 
\end{equation*}
Taking the expectation with respect to $Y\sim\calN(\mu,\sigma^2-c)$ which is absolutely continuous w.r.t. the Lebesgue measure, 
\begin{equation*}
\E_Y[\sigma^2f''(Y)-(Y-\mu)f'(Y)-f(Y)]=\E_Y[h'(Y)].
\end{equation*}
Meanwhile, applying Stein's lemma (Lemma~\ref{lemma:stein_basic}) to the random variable $Y$ and the function $f'(x)$ gives (the integrability requirement of Lemma~\ref{lemma:stein_basic} follows from bounds on the Stein factors, Lemma~\ref{lemma:stein_magnitude})
\begin{equation*}
(\sigma^2-c)\E_Y[f''(Y)]=\E_Y[(Y-\mu)f'(Y)].
\end{equation*}
Combining it with the above and with $u(c)=\E_Y[f(Y)]$,
\begin{equation*}
\E_Y[cf''(Y)]-u(c)=\E_Y[h'(Y)].
\end{equation*}
Plugging in $c=0$ gives $u(0)=-\E_Z[h'(\mu+\sigma Z)]$.

\textit{Step 2.} The second step is to consider the derivative $u'(c)$. Differentiating the definition of $u(c)=\E_Z[f(\mu+\sqrt{\sigma^2-c}Z)]$ gives
\begin{equation*}
u'(c)=-\frac{1}{2\sqrt{\sigma^2-c}}\E_Z\spar{Zf'\rpar{\mu+\sqrt{\sigma^2-c}Z}}.
\end{equation*}
Applying Stein's lemma (Lemma~\ref{lemma:stein_basic}) again to the random variable $Z$ and the function $f'(\mu+\sqrt{\sigma^2-c}x)$, where the integrability requirement again follows from Lemma~\ref{lemma:stein_magnitude},
\begin{equation*}
\sqrt{\sigma^2-c}\cdot\E_Z\spar{f''\rpar{\mu+\sqrt{\sigma^2-c}Z}}=\E_Z\spar{Zf'\rpar{\mu+\sqrt{\sigma^2-c}Z}},
\end{equation*}
therefore $u'(c)=-\half\E_Y[f''(Y)]$ and thus we obtain the ODE
\begin{equation*}
2cu'(c)+u(c)=-\E_Y[h'(Y)].
\end{equation*}

\textit{Step 3.} To solve this ODE, we reformulate it into
\begin{equation*}
\rpar{\sqrt{c}u(c)}'=-\frac{1}{2\sqrt{c}}\E_Y[h'(Y)]=-\frac{1}{2\sqrt{c}}\E_Z\spar{h'\rpar{\mu+\sqrt{\sigma^2-c}Z}},
\end{equation*}
and integrating it results in Eq.\eqref{eq:u_of_c}.
\end{proof}

\subsection{General Regret and Total Loss Bounds}\label{subsection:omitted_proofs_general}

We first restate Theorem~\ref{theorem:loss_lower} and prove it as follows. 

\lossLower*

\begin{proof}[Proof of Theorem~\ref{theorem:loss_lower}]
The proof is based on a nonasymptotic Wasserstein CLT for iid sums. Recall from Section~\ref{section:introduction} that $\rs(n)$ denotes the distribution of the sum of $n$ independent Rademacher random variables. 

First, for any function $h:\R\rightarrow\R$, no OLO algorithm can guarantee
\begin{equation*}
\loss_T<-h\rpar{\sum_{t=1}^Tg_t}+\E_{X\sim\rs(T)}[h(X)]
\end{equation*}
against all Boolean adversaries. To see this, notice that against iid Rademacher $g_{1:T}$, taking the expectation of this bound would lead to the $0<0$ contradiction. 

Next, by Lemma~\ref{lemma:wasserstein_iid} which is due to \citep{ross2011fundamentals}, there exists an absolute constant $c>0$ such that for any $1$-Lipschitz function $h:\R\rightarrow\R$ we have
\begin{equation*}
\abs{\E_{X\sim\rs(T)}[h(X)]-\E_{Z}[h(\sqrt{T}Z)]}< c.
\end{equation*}
Combining it with the above impossibility argument completes the proof.
\end{proof}

The following two corollaries are the regret analogues of Theorem~\ref{theorem:main} and Theorem~\ref{theorem:loss_lower}. The proofs are due to the loss-regret duality (Lemma~\ref{lemma:duality}) therefore omitted.  

\begin{corollary}[Upper bound on $\reg_T(u)$]\label{corollary:regret}
Algorithm~\ref{algorithm:main} with $\rho_T=0$ guarantees the regret bound
\begin{equation*}
\reg_T(u)\leq h^*(-u)+\E_{Z}[h(\rho_0 Z)]+\sum_{t=1}^T\rpar{\sqrt{\frac{2}{\pi}}\frac{\max\{g_t^2-c_t,0\}}{\rho_{t-1}}+\frac{2c_t\abs{g_t}+\abs{g_t^3}}{\rho_{t-1}^2}},\quad\forall u\in[-1,1].
\end{equation*}
\end{corollary}

\begin{corollary}[Lower bound on $\reg_T(u)$]\label{corollary:regret_lower}
There exists an absolute constant $c>0$ such that the following statement holds. For any OLO algorithm and any convex and $1$-Lipschitz function $h:\R\rightarrow\R$, there exists a Boolean adversary ($g_t\in\{-1,1\}$ for all $t$) and a comparator $u\in[-1,1]$ such that
\begin{equation*}
\reg_T(u)>h^*(-u)+\E_Z[h(\sqrt{T}Z)]-c.
\end{equation*}
\end{corollary}

\subsection{Special Cases}\label{subsection:omitted_proofs_cases}

Below we restate and prove the specializations of Corollary~\ref{corollary:regret} in Section~\ref{section:dominate_ogd}, Appendix~\ref{section:dominate_mwu} and Appendix~\ref{section:soft_thresholded}. 

\absoluteProp*

\begin{proof}[Proof of Corollary~\ref{corollary:regret_absolute}]
Starting from Corollary~\ref{corollary:regret}, it can be verified that $h^*(-u)=-h(0)=0$ for all $u\in[-1,1]$, and $\E_Z[h(\sqrt{T}Z)]=\sqrt{T}\E_Z[|Z|]=\sqrt{\frac{2}{\pi}T}$.
\end{proof}

\huberProp*

\begin{proof}[Proof of Corollary~\ref{corollary:regret_huber}]
Starting from Corollary~\ref{corollary:regret},
\begin{equation*}
h^*(-u)=-ux-h(x)\Big|_{x=-\frac{u}{\eta}}=\frac{u^2}{2\eta}.
\end{equation*}

For all $v>0$, 
\begin{align*}
\E_Z[h(vZ)]&=\int_{-\infty}^{-(v\eta)^{-1}}\rpar{-vz-\frac{1}{2\eta}}\phi(z)\diff z+\frac{\eta}{2}\int_{-(v\eta)^{-1}}^{(v\eta)^{-1}}v^2z^2\phi(z)\diff z+\int_{(v\eta)^{-1}}^\infty\rpar{vz-\frac{1}{2\eta}}\phi(z)\diff z\\
&=2v\phi\rpar{\frac{1}{v\eta}}-\frac{1}{\eta}\Phi\rpar{-\frac{1}{v\eta}}+\frac{\eta}{2}v^2\spar{\Phi(z)-z\phi(z)}\bigg|_{z=-(v\eta)^{-1}}^{(v\eta)^{-1}}\\
&=v^2\eta\spar{\Phi\rpar{\frac{1}{v\eta}}-\half}+v\phi\rpar{\frac{1}{v\eta}}-\frac{1}{\eta}\spar{1-\Phi\rpar{\frac{1}{v\eta}}}.
\end{align*}
Therefore with $v\leftarrow\sqrt{T}$ and $\eta=\frac{\alpha}{\sqrt{T}}$, 
\begin{align*}
\reg_T(u)&\leq \frac{u^2}{2\eta}+\eta T\spar{\Phi\rpar{\frac{1}{\eta\sqrt{T}}}-\half}+\sqrt{T}\phi\rpar{\frac{1}{\eta\sqrt{T}}}-\frac{1}{\eta}\spar{1-\Phi\rpar{\frac{1}{\eta\sqrt{T}}}}+O(\log T)\\
&=\spar{\frac{u^2}{2\alpha}+\rpar{\alpha+\frac{1}{\alpha}} \Phi\rpar{\frac{1}{\alpha}}+\phi\rpar{\frac{1}{\alpha}}-\half\alpha-\frac{1}{\alpha}}\sqrt{T}+O(\log T).\qedhere
\end{align*}
\end{proof}

\lseProp*

\begin{proof}[Proof of Corollary~\ref{corollary:regret_logcosh}]
Starting from Corollary~\ref{corollary:regret}, we first consider $u\in(-1,1)$. 
\begin{align*}
h^*(-u)&=-ux-\frac{1}{\eta}\ln(\cosh(\eta x))\Big|_{x=\frac{1}{\eta}\tanh^{-1}(-u)}\\
&=\frac{1}{\eta}\spar{-u\tanh^{-1}(-u)-\ln\frac{1}{\sqrt{1-u^2}}}\tag{$\cosh(\tanh^{-1}(-u))=\frac{1}{\sqrt{1-u^2}}$}\\
&=\frac{1}{2\eta}\spar{- u\ln\frac{1-u}{1+u}+\ln(1-u^2)}\tag{$\tanh^{-1}(-u)=\half\ln\frac{1-u}{1+u}$}\\
&=\frac{1}{2\eta}\spar{(1+u)\ln(1+u)+(1-u)\ln(1-u)}.
\end{align*}
The case of $u=\pm 1$ follows from $\lim_{x\rightarrow 0_+}x\ln x=0$. Further plugging the definition of $h$ into $\E_Z[h(\sqrt{T}Z)]$ completes the proof. 
\end{proof}

\softProp*

\begin{proof}[Proof of Corollary~\ref{corollary:regret_soft}]
Starting from Corollary~\ref{corollary:regret},
\begin{equation*}
h^*(-u)=\sup_{x\in\R}-ux-\max\{\abs{x}-\eta^{-1},0\}=-ux\Big|_{x=-\eta^{-1}\sign(u)}=\eta^{-1}\abs{u}.
\end{equation*}

For all $v>0$, 
\begin{align*}
\E_Z[h(vZ)]&=\int_{-\infty}^{-(v\eta)^{-1}}\rpar{-vz-\frac{1}{\eta}}\phi(z)\diff z+\int_{(v\eta)^{-1}}^\infty\rpar{vz-\frac{1}{\eta}}\phi(z)\diff z\\
&=2v\phi\rpar{\frac{1}{v\eta}}-\frac{2}{\eta}\spar{1-\Phi\rpar{\frac{1}{v\eta}}}.
\end{align*}
Substituting $v\leftarrow\sqrt{T}$ and $\eta=\frac{\alpha}{\sqrt{T}}$ completes the proof.
\end{proof}

\subsection{Proofs for the Two-Point Tradeoff}\label{subsection:omitted_two_point}

Below we restate and prove the results from Appendix~\ref{section:soft_thresholded}. 

\equivalence*

\begin{proof}[Proof of Lemma~\ref{lemma:regret_equivalence}]
Let us drop $\mathrm{alg}$ and $\mathrm{adv}$ from the notations. By definition, $\loss_T=\reg_T(0)$ and $\reg_T^\unif=\max_{u\in\{-1,1\}}\reg_T(u)$, therefore the second condition in the lemma trivially implies the first one. Now consider going from the first to the second:
\begin{align*}
\reg_T(u)&=\sum_{t=1}^Tg_t(x_t-u)\\
&=(1-\abs{u})\sum_{t=1}^Tg_tx_t+\abs{u}\sum_{t=1}^Tg_t(x_t-\sign(u))\\
&\leq (1-\abs{u})a+\abs{u}b.\qedhere
\end{align*}
\end{proof}

\twopoint*

\begin{proof}[Proof of Theorem~\ref{theorem:gamma}]
Throughout this proof, let $c>0$ be a sufficiently large absolute constant. We only prove matching upper and lower bounds on $\gamma(\eps)$, since the corresponding algorithm can be simply obtained from Corollary~\ref{corollary:regret_soft}.

An upper bound on $\gamma(\eps)$ is immediate: with the hyperparameter $\alpha$ satisfying $\gamma_{\mathrm{STh}}(0,\alpha)\sqrt{T}+c\log T=\eps\sqrt{T}$, applying Corollary~\ref{corollary:regret_soft} yields
\begin{align*}
\gamma(\eps)&=\limsup_{T\rightarrow\infty}\frac{1}{\sqrt{T}}\inf_{\mathrm{alg}}\left\{\sup_{\mathrm{adv}}\reg_T^\unif(\mathrm{alg},\mathrm{adv});\sup_{\mathrm{adv}}\loss_T(\mathrm{alg},\mathrm{adv})\leq\eps\sqrt{T}\right\}\\
&\leq\lim_{T\rightarrow\infty}\rpar{\gamma_{\mathrm{STh}}(1,\alpha)+c\frac{\log T}{\sqrt{T}}}\Big|_{\gamma_{\mathrm{STh}}(0,\alpha)=\eps-cT^{-1/2}\log T}\\
&=\gamma_{\mathrm{STh}}(1,\alpha)\Big|_{\gamma_{\mathrm{STh}}(0,\alpha)=\eps}.
\end{align*}

As for the lower bound, applying Corollary~\ref{corollary:regret_lower} on the $h$ function from Corollary~\ref{corollary:regret_soft} shows that for any algorithm $\mathrm{alg}$, there exists a Boolean adversary $\mathrm{adv}$ and a comparator $u\in[-1,1]$ such that
\begin{equation*}
\reg_T(u,\mathrm{alg},\mathrm{adv})>\gamma_{\mathrm{STh}}(u,\alpha)\sqrt{T}-c.
\end{equation*}
By Lemma~\ref{lemma:regret_equivalence}, this is equivalent to the violation of either
\begin{equation*}
\sup_{\mathrm{adv}}\loss_T(\mathrm{alg},\mathrm{adv})\leq \gamma_{\mathrm{STh}}(0,\alpha)\sqrt{T}-c
\end{equation*}
or
\begin{equation*}
\sup_{\mathrm{adv}}\reg_T^\unif(\mathrm{alg},\mathrm{adv})\leq \gamma_{\mathrm{STh}}(1,\alpha)\sqrt{T}-c.
\end{equation*}
Therefore, 
\begin{align*}
\gamma(\eps)&=\limsup_{T\rightarrow\infty}\frac{1}{\sqrt{T}}\inf_{\mathrm{alg}}\left\{\sup_{\mathrm{adv}}\reg_T^\unif(\mathrm{alg},\mathrm{adv});\sup_{\mathrm{adv}}\loss_T(\mathrm{alg},\mathrm{adv})\leq\eps\sqrt{T}\right\}\\
&\geq\lim_{T\rightarrow\infty}\rpar{\gamma_{\mathrm{STh}}(1,\alpha)-\frac{c}{\sqrt{T}}}\Big|_{\gamma_{\mathrm{STh}}(0,\alpha)=\eps+cT^{-1/2}}\\
&=\gamma_{\mathrm{STh}}(1,\alpha)\Big|_{\gamma_{\mathrm{STh}}(0,\alpha)=\eps}.
\end{align*}

Combining the upper and lower bounds, we have $\gamma(\eps)=\gamma_{\mathrm{STh}}(1,\alpha)\big|_{\gamma_{\mathrm{STh}}(0,\alpha)=\eps}$. Finally applying Lemma~\ref{lemma:gamma_two_point} completes the proof. 
\end{proof}

\subsection{Proofs of Technical Lemmas}\label{subsection:omitted_technical}

Our analysis uses several technical lemmas proved below. 

\begin{lemma}\label{lemma:expectation_lse}
For all $\alpha\in\R_{>0}$, let $f(\alpha)\defeq\E_{Z}\spar{\alpha^{-1}\ln(\cosh(\alpha Z))}$. Then, $f(\alpha)<\half\alpha$ and $\lim_{\alpha\rightarrow\infty}f(\alpha)=\sqrt{\frac{2}{\pi}}$. 
\end{lemma}

\begin{proof}[Proof of Lemma~\ref{lemma:expectation_lse}]
Starting from the first part of the lemma, applying the fact of $\cosh(x)<\exp\rpar{\frac{x^2}{2}}$ for all $x\in\R\backslash\{0\}$ gives
\begin{equation*}
f(\alpha)<\E_{Z}\spar{\alpha^{-1}\ln\rpar{\exp\rpar{\half\alpha^2 Z^2}}}=\half\alpha\E_Z[Z^2]=\half\alpha.
\end{equation*}

Next, for all $x\in\R$,
\begin{equation*}
\alpha^{-1}\ln(\cosh(\alpha x))=\abs{x}+\alpha^{-1}\ln\frac{1+\exp(-2\alpha\abs{x})}{2}. 
\end{equation*}
Denote the second term on the RHS as $g(\alpha,x)$, which satisfies the pointwise convergence $\lim_{\alpha\rightarrow\infty}g(\alpha,x)=0$. Furthermore, for all $\alpha\geq 1$ we have $\abs{g(\alpha,x)}\leq\alpha^{-1}\ln 2\leq \ln 2$, therefore by the dominated convergence theorem, $\lim_{\alpha\rightarrow\infty}\E_Z[g(\alpha,Z)]=0$. This leads to
\begin{equation*}
\lim_{\alpha\rightarrow\infty}f(\alpha)=\E_Z[\abs{Z}]=\sqrt{\frac{2}{\pi}}.\qedhere
\end{equation*}
\end{proof}

\begin{lemma}\label{lemma:gamma_two_point}
For the quantity $\gamma_{\mathrm{STh}}(u,\alpha)$ defined in Corollary~\ref{corollary:regret_soft}, we have
\begin{equation*}
\int_{-\infty}^{\gamma_{\mathrm{STh}}(0,\alpha)-\gamma_{\mathrm{STh}}(1,\alpha)}\Phi(x)\diff x=\frac{1}{2}\gamma_{\mathrm{STh}}(0,\alpha).
\end{equation*}
\end{lemma}

\begin{proof}[Proof of Lemma~\ref{lemma:gamma_two_point}]
Recall that
\begin{equation*}
\gamma_{\mathrm{STh}}(u,\alpha)=\frac{\abs{u}}{\alpha}+\frac{2}{\alpha}\Phi\rpar{\frac{1}{\alpha}}+2\phi\rpar{\frac{1}{\alpha}}-\frac{2}{\alpha},
\end{equation*}
therefore $\gamma_{\mathrm{STh}}(0,\alpha)-\gamma_{\mathrm{STh}}(1,\alpha)=-\alpha^{-1}$. It can be simply verified that for all $x\in\R$, $\int_{-\infty}^x\Phi(z)\diff z=x\Phi(x)+\phi(x)$. Therefore
\begin{equation*}
\int_{-\infty}^{-\alpha^{-1}}\Phi(x)\diff x=-\frac{1}{\alpha}\Phi\rpar{-\frac{1}{\alpha}}+\phi\rpar{-\frac{1}{\alpha}}=-\frac{1}{\alpha}\spar{1-\Phi\rpar{\frac{1}{\alpha}}}+\phi\rpar{\frac{1}{\alpha}}=\frac{1}{2}\gamma_{\mathrm{STh}}(0,\alpha).\qedhere
\end{equation*}
\end{proof}

\begin{lemma}\label{lemma:gap_soft}
For all $\eps\in(0,\sqrt{\frac{2}{\pi}}]$, $\gamma(\eps)$ defined in Eq.\eqref{eq:gamma} satisfies
\begin{equation*}
\gamma(\eps)-\eps<\sqrt{2}\cdot\erfi^{-1}\rpar{\frac{\sqrt{2}}{\sqrt{\pi}\eps}}.
\end{equation*}
\end{lemma}

\begin{proof}[Proof of Lemma~\ref{lemma:gap_soft}]
It suffices to prove
\begin{equation*}
\int_{-\infty}^{\eps-\gamma(\eps)}\Phi(x)\diff x>\int_{-\infty}^{-\sqrt{2}\cdot\erfi^{-1}\rpar{\frac{\sqrt{2}}{\sqrt{\pi}\eps}}}\Phi(x)\diff x,
\end{equation*}
where $\lhs=\frac{\eps}{2}$ by Theorem~\ref{theorem:gamma}. By a change of variable $z=\erfi^{-1}\rpar{\frac{\sqrt{2}}{\sqrt{\pi}\eps}}$, this is further equivalent to proving
\begin{equation}\label{eq:gap_soft_partial}
\erfi(z)\int_{-\infty}^{-\sqrt{2}z}\Phi(x)\diff x<\sqrt{\frac{2}{\pi}},\quad\forall z\geq \erfi^{-1}(1).
\end{equation}

To this end, notice that for all $x<0$ we have
\begin{align*}
\int_{-\infty}^x\Phi(z)\diff z&=x\Phi(x)+\phi(x)\\
&=x[1-\Phi(-x)]+\phi(-x)\\
&<2\phi(-x). \tag{Mills ratio bound; Lemma~\ref{lemma:mills}}
\end{align*}
Therefore the LHS of Eq.\eqref{eq:gap_soft_partial} is upper-bounded by
\begin{equation*}
\erfi(z)\int_{-\infty}^{-\sqrt{2}z}\Phi(x)\diff x<2\erfi(z)\phi(\sqrt{2}z)=\sqrt{\frac{2}{\pi}}\exp(-z^2)\int_{0}^z\exp(x^2)\diff x\leq \sqrt{\frac{2}{\pi}}.\qedhere
\end{equation*}
\end{proof}

\subsection{Summary of Existing Lemmas}\label{subsection:omitted_existing}

This subsection summarizes several known results from the literature, whose proofs are omitted. 

The first one is a formal version of the loss-regret duality \citep[Theorem~10.6]{orabona2025modern} used throughout our analysis. Notice that if the function $\psi$ has a convex and compact domain, then the following lemma can still be applied by assigning $\psi$ as $\infty$ outside its domain, which is standard.  

\begin{lemma}[Loss-regret duality]\label{lemma:duality}
Let $\psi:\R\rightarrow(-\infty,\infty]$ be a proper, closed and convex function, and let $\psi^*$ be its convex conjugate. For any two sequences $x_1,\ldots,x_T\in\R$ and $g_1,\ldots,g_T\in\R$, we have
\begin{equation*}
\sum_{t=1}^Tg_tx_t\leq -\psi^*\rpar{-\sum_{t=1}^Tg_t}
\end{equation*}
if and only if
\begin{equation*}
\sum_{t=1}^Tg_t(x_t-u)\leq \psi(u),\quad\forall u\in\R. 
\end{equation*}
\end{lemma}

The next lemma is a Wasserstein normal approximation bound for iid Rademacher sums, simplified from \citep[Theorem~3.2]{ross2011fundamentals}. The constant is not optimized. 

\begin{lemma}[Nonasymptotic Wasserstein CLT for iid Rademacher sums]\label{lemma:wasserstein_iid}
Let $X_1,\ldots,X_n$ be iid Rademacher random variables. If $S_n=\sum_{i=1}^nX_i$ and $Z$ has the standard normal distribution, then for any $1$-Lipschitz function $f:\R\rightarrow\R$ we have
\begin{equation*}
\abs{\E[f(S_n)]-\E_Z[f(\sqrt{n}Z)]}\leq 1+\sqrt{\frac{2}{\pi}}.
\end{equation*}
\end{lemma}

The final lemma is a standard estimate of the Mills ratio, due to \citep{gordon1941values}.

\begin{lemma}[Mills ratio bound]\label{lemma:mills}
Consider the Mills ratio defined as $m(x)\defeq\frac{1-\Phi(x)}{\phi(x)}$. For all $x>0$, 
\begin{equation*}
\frac{x}{x^2+1}< m(x)< \frac{1}{x}.
\end{equation*}
\end{lemma}

\section{Closed Forms of the Output}\label{section:reformulation}

To provide additional intuition and ease the implementation of our algorithm, this section derives the closed forms of Algorithm~\ref{algorithm:main}'s output (as well as the $f_{\mu,\sigma,h}$ function), focusing on the two special cases of $h$ from Section~\ref{section:dominate_ogd}. The goal is to convert the computation of $x_t$ to querying common Gaussian-integral-type special functions. Auxiliary lemmas for the derivation are provided in Appendix~\ref{subsection:auxiliary}. 

\subsection{Absolute Value Function}

Consider the setting of Section~\ref{subsection:absolute}.

\begin{proposition}[Algorithm: absolute value]\label{proposition:absolute}
If $h(x)=\abs{x}$, then in Algorithm~\ref{algorithm:main},
\begin{equation*}
f_{\mu,\sigma,h}(x)=\begin{dcases}
1-\frac{\mu+\E_Z[\abs{\mu+\sigma Z}]}{\sigma^2}\frac{\Phi_{\mu,\sigma}(x)}{\phi_{\mu,\sigma}(x)},&x\leq 0,\\
-1-\frac{\mu-\E_Z[\abs{\mu+\sigma Z}]}{\sigma^2}\frac{1-\Phi_{\mu,\sigma}(x)}{\phi_{\mu,\sigma}(x)},&x>0.
\end{dcases}
\end{equation*}
Furthermore, the output Eq.\eqref{eq:algorithm} is equivalent to
\begin{equation}\label{eq:absolute_integral_formal}
x_{t}=1-\int_0^{1}\frac{1}{\sqrt{\tau}}\Phi\rpar{\frac{s_{t-1}}{\sqrt{\rho^2_{t-1}-\tau c_t}}}\diff\tau,
\end{equation}
which can be further expressed as one of the following three cases: 
\begin{itemize}
\item If $\rho^2_{t-1}>c_t=0$, then
\begin{equation*}
x_{t}
=1-2\Phi\rpar{\frac{s_{t-1}}{\rho_{t-1}}}.
\end{equation*}
\item If $\rho_{t-1}^2=c_t>0$ which is only possible at $t=T$ by the requirement of Algorithm~\ref{algorithm:main}, then
\begin{equation*}
x_t=-\sign(s_{t-1})\spar{1+\sqrt{\frac{\pi}{2}}\frac{\abs{s_{t-1}}}{\rho_{t-1}}-\sqrt{\frac{\pi}{2}}\frac{s_{t-1}}{\rho_{t-1}}\rpar{2\Phi\rpar{\frac{s_{t-1}}{\rho_{t-1}}}-1}-\sqrt{2\pi}\phi\rpar{\frac{s_{t-1}}{\rho_{t-1}}}}.
\end{equation*}
\item If $\rho_{t-1}^2>c_t>0$, then with Owen's T function $T(x,y)\defeq\frac{1}{2\pi}\int_0^y\frac{e^{-\half x^2(1+z^2)}}{1+z^2}\diff z$,
\begin{align}
x_{t}
&=1-2\Phi\rpar{\frac{s_{t-1}}{\sqrt{\rho_{t-1}^2-c_t}}}\nonumber\\
&\quad -\frac{\sqrt{2\pi}\rho_{t-1}}{\sqrt{c_t}}\phi\rpar{\frac{s_{t-1}}{\rho_{t-1}}}\spar{1-2\Phi\rpar{\frac{s_{t-1}}{\rho_{t-1}}\sqrt{\frac{c_t}{\rho_{t-1}^2-c_t}}}}-\frac{2\sqrt{2\pi}s_{t-1}}{\sqrt{c_t}}T\rpar{\frac{s_{t-1}}{\rho_{t-1}},\sqrt{\frac{c_t}{\rho_{t-1}^2-c_t}}}.\label{eq:absolute_general_case}
\end{align}
\end{itemize}
For implementation, both the normal CDF and the Owen's T function can be queried from standard software packages, such as SciPy. 
\end{proposition}

Several remarks are in order.
\begin{itemize}
\item To make sense of these somewhat complicated expressions, first notice that all the expressions are consistent with the standard dimensional analysis, where $\mu$, $\sigma$ and $\sqrt{c}$ sharing the same ``unit''. 
\item Next, the case of $\rho_{t-1}^2>c_t=0$ is the most interpretable one. It can be recovered from the general case of $\rho_{t-1}^2>c_t>0$: as $c_t\rightarrow 0$, the second line of Eq.\eqref{eq:absolute_general_case} vanishes, therefore $x_t\rightarrow 1-2\Phi(s_{t-1}/\rho_{t-1})$. This equals the probability of a standard Brownian motion starting at $s_{t-1}$ ending up being negative after $\rho^2_{t-1}$ time. Also see Section~\ref{section:algorithm} for discussions of this continuous time approximation. 
\item As shown below, the ``last round special case'' $\rho_{t-1}^2=c_t>0$ and the general case $\rho_{t-1}^2>c_t>0$ are proved using different strategies: the former is based on the closed form of $f_{\mu,\sigma,h}$ while the latter is based on the integral Eq.\eqref{eq:absolute_integral_formal}. Therefore a simple sanity check is to verify that the latter recovers the former in the limit of $c_t\rightarrow\rho_{t-1}^2$, which follows from $T(x,\infty)=\half[1-\Phi(\abs{x})]$ \citep[Eq.(2.4) in Table II]{owen1980table}.
\end{itemize} 

\begin{proof}[Proof of Proposition~\ref{proposition:absolute}]
The proof is divided into three parts: the first part analyzes the function $f_{\mu,\sigma,h}$, the second part derives the general integral representation Eq.\eqref{eq:absolute_integral_formal} of $x_t$, and the third part relates $x_t$ to common special functions. 

\paragraph{Part 1} Starting from the representation Eq.\eqref{eq:standard_rep} of the solution of Stein equation, for all $x\leq 0$, 
\begin{align*}
f_{\mu,\sigma,h}(x)&=\frac{1}{\sigma^2\phi_{\mu,\sigma}(x)}\int_{-\infty}^x(-z-\E_Z[\abs{\mu+\sigma Z}])\phi_{\mu,\sigma}(z)\diff z\\
&=\frac{1}{\sigma^2\phi_{\mu,\sigma}(x)}\rpar{\int_{-\infty}^x(-z+\mu)\phi_{\mu,\sigma}(z)\diff z-(\mu+\E_Z[\abs{\mu+\sigma Z}])\int_{-\infty}^x\phi_{\mu,\sigma}(z)\diff z}\\
&=\frac{1}{\sigma^2\phi_{\mu,\sigma}(x)}\rpar{\sigma^2\int_{-\infty}^x\phi'_{\mu,\sigma}(z)\diff z-(\mu+\E_Z[\abs{\mu+\sigma Z}])\int_{-\infty}^x\phi_{\mu,\sigma}(z)\diff z}\\
&=\frac{1}{\sigma^2\phi_{\mu,\sigma}(x)}\rpar{\sigma^2\phi_{\mu,\sigma}(x)-(\mu+\E_Z[\abs{\mu+\sigma Z}])\Phi_{\mu,\sigma}(x)}\\
&=1-\frac{\mu+\E_Z[\abs{\mu+\sigma Z}]}{\sigma^2}\frac{\Phi_{\mu,\sigma}(x)}{\phi_{\mu,\sigma}(x)}.
\end{align*}
The other case of $x>0$ is similar,
\begin{align*}
f_{\mu,\sigma,h}(x)&=-\frac{1}{\sigma^2\phi_{\mu,\sigma}(x)}\int_{x}^{\infty}(z-\E_Z[\abs{\mu+\sigma Z}])\phi_{\mu,\sigma}(z)\diff z\\
&=-\frac{1}{\sigma^2\phi_{\mu,\sigma}(x)}\rpar{-\int_{x}^{\infty}(-z+\mu)\phi_{\mu,\sigma}(z)\diff z+(\mu-\E_Z[\abs{\mu+\sigma Z}])\int_{x}^{\infty}\phi_{\mu,\sigma}(z)\diff z}\\
&=-\frac{1}{\sigma^2\phi_{\mu,\sigma}(x)}\rpar{-\sigma^2\int_{x}^{\infty}\phi'_{\mu,\sigma}(z)\diff z+(\mu-\E_Z[\abs{\mu+\sigma Z}])\int_{x}^{\infty}\phi_{\mu,\sigma}(z)\diff z}\\
&=-\frac{1}{\sigma^2\phi_{\mu,\sigma}(x)}\rpar{\sigma^2\phi_{\mu,\sigma}(x)+(\mu-\E_Z[\abs{\mu+\sigma Z}])\rpar{1-\Phi_{\mu,\sigma}(x)}}\\
&=-1-\frac{\mu-\E_Z[\abs{\mu+\sigma Z}]}{\sigma^2}\frac{1-\Phi_{\mu,\sigma}(x)}{\phi_{\mu,\sigma}(x)}.
\end{align*}

\paragraph{Part 2} Next we prove Eq.\eqref{eq:absolute_integral_formal}. By Lemma~\ref{lemma:equivalent}, $x_t$ can be represented as the following integral with $\mu\leftarrow s_{t-1}$, $\sigma\leftarrow\rho_{t-1}>0$ and $c\leftarrow c_t>0$,
\begin{equation*}
x_{t}=-\frac{1}{2}\int_0^{1}\frac{1}{\sqrt{\tau}}\E_Z\spar{\sign\rpar{\mu+\sqrt{\sigma^2-\tau c}Z}}\diff\tau.
\end{equation*}
For the integrand, it is valid to only consider $\tau<1$ and obtain
\begin{equation*}
\E_Z\spar{\sign\rpar{\mu+\sqrt{\sigma^2-\tau c}Z}}=1-2\P\rpar{\mu+\sqrt{\sigma^2-\tau c}Z\leq 0}=2\Phi\rpar{\frac{\mu}{\sqrt{\sigma^2-\tau c}}}-1.
\end{equation*}
Therefore combining them gives Eq.\eqref{eq:absolute_integral_formal} stated in the proposition. 

\paragraph{Part 3} The final step is to relate Eq.\eqref{eq:absolute_integral_formal} to common special functions. The case of $\rho^2_{t-1}>c_t=0$ simply follows from Eq.\eqref{eq:absolute_integral_formal}. The general case of $\rho_{t-1}^2>c_t\geq 0$ follows from Lemma~\ref{lemma:special_integral}, which is separated for reusability. Below we only consider the last round special case, $\rho_{t-1}^2=c_t>0$. 

Let $\mu\leftarrow s_{t-1}$ and $\sigma\leftarrow\rho_{t-1}>0$. The definition of $x_t$ in Eq.\eqref{eq:algorithm} and the above results on $f_{\mu,\sigma,h}$ give
\begin{equation*}
x_{t}=f_{\mu,\sigma,h}(\mu)=\begin{dcases}
1-\frac{\mu+\E_Z[\abs{\mu+\sigma Z}]}{\sigma^2}\frac{\Phi_{\mu,\sigma}(\mu)}{\phi_{\mu,\sigma}(\mu)},&\mu\leq 0,\\
-1-\frac{\mu-\E_Z[\abs{\mu+\sigma Z}]}{\sigma^2}\frac{1-\Phi_{\mu,\sigma}(\mu)}{\phi_{\mu,\sigma}(\mu)},&\mu>0,
\end{dcases}
\end{equation*}
where $\Phi_{\mu,\sigma}(\mu)=\half$ and $\phi_{\mu,\sigma}(\mu)=\frac{1}{\sqrt{2\pi}\sigma}$. These two cases on $\mu$ can be combined into
\begin{equation*}
x_t=-\sign(\mu)\spar{1+\frac{\sqrt{2\pi}}{2\sigma}\rpar{\abs{\mu}-\E_Z[\abs{\mu+\sigma Z}]}},
\end{equation*}
where
\begin{equation*}
\E_Z[\abs{\mu+\sigma Z}]=\int_{-\infty}^{-\frac{\mu}{\sigma}}(-\mu-\sigma x)\phi(x)\diff x+\int_{-\frac{\mu}{\sigma}}^\infty(\mu+\sigma x)\phi(x)\diff x=\mu\spar{2\Phi\rpar{\frac{\mu}{\sigma}}-1}+2\sigma\phi\rpar{\frac{\mu}{\sigma}}.\qedhere
\end{equation*}
\end{proof}

\subsection{Huber Function}

Consider the setting of Section~\ref{subsection:huber}.

\begin{proposition}[Algorithm: Huber]\label{proposition:huber}
If $h(x)=\frac{k}{2}x^2\cdot\bm{1}[\abs{x}\leq k^{-1}]+(\abs{x}-\frac{1}{2k})\cdot\bm{1}[\abs{x}> k^{-1}]$ for some $k\in\R_{>0}$, then in Algorithm~\ref{algorithm:main},
\begin{equation*}
f_{\mu,\sigma,h}(x)=\begin{dcases}
1-\frac{\mu+\E_Z[h(\mu+\sigma Z)]+(2k)^{-1}}{\sigma^2}\frac{\Phi_{\mu,\sigma}(x)}{\phi_{\mu,\sigma}(x)},&x<-k^{-1},\\
\begin{aligned}
\frac{1}{2\sigma^2\phi_{\mu,\sigma}(x)}\bigg[(k\mu^2+k\sigma^2-2\E_Z[h(\mu+\sigma Z)])\Phi_{\mu,\sigma}(x)-k\sigma^2(x+\mu)\phi_{\mu,\sigma}(x)\\-(k\mu^2+k\sigma^2+2\mu+k^{-1})\Phi_{\mu,\sigma}(-k^{-1})+\sigma^2(1+k\mu )\phi_{\mu,\sigma}(-k^{-1})\bigg],
\end{aligned}
&-k^{-1}\leq x\leq k^{-1},\\
-1-\frac{\mu-\E_Z[h(\mu+\sigma Z)]-(2k)^{-1}}{\sigma^2}\frac{1-\Phi_{\mu,\sigma}(x)}{\phi_{\mu,\sigma}(x)},&x>k^{-1}.
\end{dcases}
\end{equation*}
Furthermore, the output Eq.\eqref{eq:algorithm} is equivalent to
\begin{align}
x_{t}=1-\frac{1}{2}\int_0^{1}\frac{1}{\sqrt{\tau}}\Bigg[(1-ks_{t-1})\Phi\rpar{\frac{s_{t-1}-k^{-1}}{\sqrt{\rho_{t-1}^2-\tau c_t}}}+(1+ks_{t-1})\Phi\rpar{\frac{s_{t-1}+k^{-1}}{\sqrt{\rho_{t-1}^2-\tau c_t}}}\nonumber\\
-k\sqrt{\rho_{t-1}^2-\tau c_t}\rpar{\phi\rpar{\frac{s_{t-1}-k^{-1}}{\sqrt{\rho_{t-1}^2-\tau c_t}}}-\phi\rpar{\frac{s_{t-1}+k^{-1}}{\sqrt{\rho_{t-1}^2-\tau c_t}}}}\Bigg]\diff\tau.\label{eq:huber_integral}
\end{align}
\end{proposition}

Slightly different from Proposition~\ref{proposition:absolute}, in Proposition~\ref{proposition:huber} we leave the output $x_t$ as a general integral Eq.\eqref{eq:huber_integral} for brevity. Nonetheless, the same idea still holds: Eq.\eqref{eq:huber_integral} can be further converted to a closed form using Gaussian-integral-type special functions. Concretely, 
\begin{itemize}
\item If $\rho^2_{t-1}>c_t=0$, then
\begin{align*}
x_{t}&= 1-\Phi\rpar{\frac{s_{t-1}-k^{-1}}{\rho_{t-1}}}-\Phi\rpar{\frac{s_{t-1}+k^{-1}}{\rho_{t-1}}}\\
&\quad +k\spar{s_{t-1}\rpar{\Phi\rpar{\frac{s_{t-1}-k^{-1}}{\rho_{t-1}}}-\Phi\rpar{\frac{s_{t-1}+k^{-1}}{\rho_{t-1}}}}+\rho_{t-1}\rpar{\phi\rpar{\frac{s_{t-1}-k^{-1}}{\rho_{t-1}}}-\phi\rpar{\frac{s_{t-1}+k^{-1}}{\rho_{t-1}}}}}.
\end{align*}
\item If $\rho_{t-1}^2=c_t>0$ which is only possible at $t=T$ by the requirement of Algorithm~\ref{algorithm:main}, then $x_t=f_{s_{t-1},\rho_{t-1},h}(s_{t-1})$. The closed form of $f_{\mu,\sigma,h}$ is derived in Proposition~\ref{proposition:huber}. 
\item If $\rho_{t-1}^2>c_t>0$, then it suffices to apply the Gaussian integral formulas from Lemma~\ref{lemma:special_integral} and \ref{lemma:special_integral_alt}. The final result is the closed form of $x_t$ based on Owen's T function. 
\end{itemize}

\begin{remark}[Effective learning rate]\label{remark:effective_lr}
Intuition can be obtained by reasoning about the effective learning rate of Algorithm~\ref{algorithm:main}. Suppose $\rho^2_{t}>c_{t+1}=0$ for some $t\in[0:T-1]$, giving
\begin{align*}
x_{t+1}&= 1-\Phi\rpar{\frac{s_{t}-k^{-1}}{\rho_{t}}}-\Phi\rpar{\frac{s_{t}+k^{-1}}{\rho_{t}}}\\
&\quad +k\spar{s_{t}\rpar{\Phi\rpar{\frac{s_{t}-k^{-1}}{\rho_{t}}}-\Phi\rpar{\frac{s_{t}+k^{-1}}{\rho_{t}}}}+\rho_{t}\rpar{\phi\rpar{\frac{s_{t}-k^{-1}}{\rho_{t}}}-\phi\rpar{\frac{s_{t}+k^{-1}}{\rho_{t}}}}}.
\end{align*}
Denote the RHS as $\tilde x_{t+1}$, the output of the $c_t$-independent analogue of Algorithm~\ref{algorithm:main}. Taking the derivative with respect to $s_t$ gives
\begin{align*}
\frac{\partial \tilde x_{t+1}}{\partial s_t}&= -\frac{2}{\rho_t}\phi\rpar{\frac{s_{t}-k^{-1}}{\rho_{t}}}+k\spar{\Phi\rpar{\frac{s_{t}-k^{-1}}{\rho_{t}}}-\Phi\rpar{\frac{s_{t}+k^{-1}}{\rho_{t}}}}\\
&\quad +k\spar{\frac{s_{t}}{\rho_t}\rpar{\phi\rpar{\frac{s_{t}-k^{-1}}{\rho_{t}}}-\phi\rpar{\frac{s_{t}+k^{-1}}{\rho_{t}}}}-\frac{s_t-k^{-1}}{\rho_t}\phi\rpar{\frac{s_{t}-k^{-1}}{\rho_{t}}}+\frac{s_t+k^{-1}}{\rho_t}\phi\rpar{\frac{s_{t}+k^{-1}}{\rho_{t}}}}.
\end{align*}
Furthermore, $\tilde x_{t+1}=0$ when $s_t=0$. 

Now suppose $s_{t-1}=0$ therefore $s_t=g_t$ and $\tilde x_t=0$. Based on the above, we linearize $\tilde x_{t+1}$ with respect to $g_t$ (also $s_t$, as $s_t=g_t$) and obtain
\begin{align*}
\tilde x_{t+1}-\tilde x_t&=\tilde x_{t+1}\\
&=\tilde x_{t+1}\Big|_{s_t=0}+\frac{\partial \tilde x_{t+1}}{\partial s_t}\Big|_{s_t=0}\cdot g_t+o(g_t)\\
&=k\spar{\Phi\rpar{-\frac{1}{k\rho_{t}}}-\Phi\rpar{\frac{1}{k\rho_{t}}}}+o(g_t)\\
&=-k\cdot \erf\rpar{\frac{1}{\sqrt{2}k\rho_t}}g_t+o(g_t).
\end{align*}
The absolute value of $g_t$'s coefficient can be regarded as the effective learning rate. 
\end{remark}

\begin{proof}[Proof of Proposition~\ref{proposition:huber}]
The proof generalizes that of Proposition~\ref{proposition:absolute}. Again it is divided into two parts analyzing $f_{\mu,\sigma,h}$ and $x_t$ separately. 

\paragraph{Part 1} Consider $f_{\mu,\sigma,h}$. Starting from the representation Eq.\eqref{eq:standard_rep} of the solution of Stein equation, for all $x< -k^{-1}$, 
\begin{align*}
f_{\mu,\sigma,h}(x)&=\frac{1}{\sigma^2\phi_{\mu,\sigma}(x)}\spar{\int_{-\infty}^x (-z+\mu)\phi_{\mu,\sigma}(z)\diff z-\rpar{\frac{1}{2k}+\mu+\E_Z[h(\mu+\sigma Z)]}\Phi_{\mu,\sigma}(x)}\\
&=\frac{1}{\sigma^2\phi_{\mu,\sigma}(x)}\spar{\sigma^2\phi_{\mu,\sigma}(x)-\rpar{\frac{1}{2k}+\mu+\E_Z[h(\mu+\sigma Z)]}\Phi_{\mu,\sigma}(x)}\\
&=1-\frac{\mu+\E_Z[h(\mu+\sigma Z)]+(2k)^{-1}}{\sigma^2}\frac{\Phi_{\mu,\sigma}(x)}{\phi_{\mu,\sigma}(x)}.
\end{align*}
Similarly, for all $x>k^{-1}$, 
\begin{align*}
f_{\mu,\sigma,h}(x)&=-\frac{1}{\sigma^2\phi_{\mu,\sigma}(x)}\spar{\int_{x}^{\infty}(z-\mu)\phi_{\mu,\sigma}(z)\diff z-\rpar{\frac{1}{2k}-\mu+\E_Z[h(\mu+\sigma Z)]}\rpar{1-\Phi_{\mu,\sigma}(x)}}\\
&=-\frac{1}{\sigma^2\phi_{\mu,\sigma}(x)}\spar{\sigma^2\phi_{\mu,\sigma}(x)-\rpar{\frac{1}{2k}-\mu+\E_Z[h(\mu+\sigma Z)]}\rpar{1-\Phi_{\mu,\sigma}(x)}}\\
&=-1-\frac{\mu-\E_Z[h(\mu+\sigma Z)]-(2k)^{-1}}{\sigma^2}\frac{1-\Phi_{\mu,\sigma}(x)}{\phi_{\mu,\sigma}(x)}.
\end{align*}
As for the regime in between: for all $x\in[-k^{-1},k^{-1}]$,
\begin{align*}
&f_{\mu,\sigma,h}(x)\\
=~&\frac{1}{\sigma^2\phi_{\mu,\sigma}(x)}\spar{\int_{-\infty}^{-k^{-1}}\rpar{-z-\frac{1}{2k}}\phi_{\mu,\sigma}(z)\diff z+\int_{-k^{-1}}^{x}\frac{k}{2}z^2\phi_{\mu,\sigma}(z)\diff z-\E_Z[h(\mu+\sigma Z)]\Phi_{\mu,\sigma}(x)}\\
=~&\frac{1}{\sigma^2\phi_{\mu,\sigma}(x)}\spar{\sigma^2\phi_{\mu,\sigma}(-k^{-1})-\rpar{\mu+\frac{1}{2k}}\Phi_{\mu,\sigma}(-k^{-1})+\frac{k}{2}\int_{-k^{-1}}^{x}z^2\phi_{\mu,\sigma}(z)\diff z-\E_Z[h(\mu+\sigma Z)]\Phi_{\mu,\sigma}(x)},
\end{align*}
where
\begin{equation*}
\int_{-k^{-1}}^{x}z^2\phi_{\mu,\sigma}(z)\diff z=\spar{(\mu^2+\sigma^2)\Phi_{\mu,\sigma}(z)-\sigma^2(z+\mu)\phi_{\mu,\sigma}(z)}\bigg|_{z=-k^{-1}}^{x}.
\end{equation*}
Plugging it back,
\begin{align*}
f_{\mu,\sigma,h}(x)&=\frac{1}{2\sigma^2\phi_{\mu,\sigma}(x)}\bigg[(k\mu^2+k\sigma^2-2\E_Z[h(\mu+\sigma Z)])\Phi_{\mu,\sigma}(x)-k\sigma^2(x+\mu)\phi_{\mu,\sigma}(x)\\
&\quad\quad\quad -(k\mu^2+k\sigma^2+2\mu+k^{-1})\Phi_{\mu,\sigma}(-k^{-1})+\sigma^2(1+k\mu )\phi_{\mu,\sigma}(-k^{-1})\bigg].
\end{align*}

\paragraph{Part 2} Consider $x_t$. By the same substitution $\mu\leftarrow s_{t-1}$, $\sigma\leftarrow\rho_{t-1}>0$ and $c\leftarrow c_t>0$ as in Proposition~\ref{proposition:absolute}, $x_t$ can be reformulated as
\begin{equation*}
x_{t}=-\frac{1}{2}\int_0^{1}\frac{1}{\sqrt{\tau}}\E_Z\spar{h'\rpar{\mu+\sqrt{\sigma^2-\tau c}Z}}\diff\tau,
\end{equation*}
where
\begin{align*}
\E_Z\spar{h'\rpar{\mu+\sqrt{\sigma^2-\tau c}Z}}&=-\int_{-\infty}^{\frac{-k^{-1}-\mu}{\sqrt{\sigma^2-\tau c}}}\phi(z)\diff z+\int_{\frac{k^{-1}-\mu}{\sqrt{\sigma^2-\tau c}}}^\infty\phi(z)\diff z+k\int_{\frac{-k^{-1}-\mu}{\sqrt{\sigma^2-\tau c}}}^{\frac{k^{-1}-\mu}{\sqrt{\sigma^2-\tau c}}}\rpar{\mu+\sqrt{\sigma^2-\tau c}z}\phi(z)\diff z\\
&=-1+(1-k\mu)\Phi\rpar{\frac{-k^{-1}+\mu}{\sqrt{\sigma^2-\tau c}}}+(1+k\mu)\Phi\rpar{\frac{k^{-1}+\mu}{\sqrt{\sigma^2-\tau c}}}\\
&\quad\quad\quad-k\sqrt{\sigma^2-\tau c}\spar{\phi\rpar{\frac{-k^{-1}+\mu}{\sqrt{\sigma^2-\tau c}}}-\phi\rpar{\frac{k^{-1}+\mu}{\sqrt{\sigma^2-\tau c}}}}.
\qedhere
\end{align*}
\end{proof}

\subsection{Auxiliary Lemmas}\label{subsection:auxiliary}

The above derivation uses the following two Gaussian integral lemmas, derived by somewhat complicated but mechanical calculations. 

\begin{lemma}\label{lemma:special_integral}
For all $\mu\in\R$, $\sigma\in\R_{>0}$ and $c\in(0,\sigma^2)$, 
\begin{align*}
&1-\int_0^{1}\frac{1}{\sqrt{\tau}}\Phi\rpar{\frac{\mu}{\sqrt{\sigma^2-\tau c}}}\diff\tau\\
=~&1-2\Phi\rpar{\frac{\mu}{\sqrt{\sigma^2-c}}}-\frac{\sqrt{2\pi}\sigma}{\sqrt{c}}\phi\rpar{\frac{\mu}{\sigma}}\spar{1-2\Phi\rpar{\frac{\mu}{\sigma}\sqrt{\frac{c}{\sigma^2-c}}}}-\frac{2\sqrt{2\pi}\mu}{\sqrt{c}}T\rpar{\frac{\mu}{\sigma},\sqrt{\frac{c}{\sigma^2-c}}}.
\end{align*}
\end{lemma}

\begin{proof}[Proof of Lemma~\ref{lemma:special_integral}]
Denote the LHS of the objective equality as $f(\mu)$. Taking the derivative gives
\begin{align*}
f'(u)&=-\frac{1}{\sqrt{2\pi}}\int_0^1\frac{1}{\sqrt{\tau(\sigma^2-\tau c)}}\exp\rpar{-\frac{\mu^2}{2(\sigma^2-\tau c)}}\diff\tau\\
&=-\frac{2}{\sqrt{2\pi c}}\int_{0}^{\sqrt{\frac{c}{\sigma^2-c}}}\frac{1}{1+\tau^2}\exp\rpar{-\frac{\mu^2(1+\tau^2)}{2\sigma^2}}\diff\tau\tag{$\tau\leftarrow \frac{\tau^2\sigma^2}{(1+\tau^2)c}$}\\
&=-\frac{2\sqrt{2\pi}}{\sqrt{c}}T\rpar{\frac{\mu}{\sigma},\sqrt{\frac{c}{\sigma^2-c}}},\tag{$T(x,y)\defeq\frac{1}{2\pi}\int_0^y\frac{e^{-\half x^2(1+z^2)}}{1+z^2}\diff z$}
\end{align*}
where the last line uses the definition of Owen's T function. It is known that \citep[Eq.(c00,000.1) in Table I]{owen1980table}
\begin{equation*}
\int T(x,y)\diff x= xT(x,y)+\frac{y}{\sqrt{2\pi(1+y^2)}}\Phi\rpar{x\sqrt{1+y^2}}-\phi(x)\Phi(yx)+\half\phi(x).
\end{equation*}
Therefore, since $f(0)=0$,
\begin{align*}
f(\mu)&=-\frac{2\sqrt{2\pi}}{\sqrt{c}}\int_0^\mu T\rpar{\frac{x}{\sigma},\sqrt{\frac{c}{\sigma^2-c}}}\diff x\\
&=-\frac{2\sqrt{2\pi}\sigma}{\sqrt{c}}\spar{xT\rpar{x,\sqrt{\frac{c}{\sigma^2-c}}}+\frac{\sqrt{c}}{\sqrt{2\pi}\sigma}\Phi\rpar{\frac{\sigma x}{\sqrt{\sigma^2-c}}}-\phi(x)\Phi\rpar{\frac{\sqrt{c}x}{\sqrt{\sigma^2-c}}}+\half\phi(x)}\Bigg|_{x=0}^{\frac{\mu}{\sigma}}\\
&=1-2\Phi\rpar{\frac{\mu}{\sqrt{\sigma^2-c}}}-\frac{\sqrt{2\pi}\sigma}{\sqrt{c}}\phi\rpar{\frac{\mu}{\sigma}}\spar{1-2\Phi\rpar{\frac{\mu}{\sigma}\sqrt{\frac{c}{\sigma^2-c}}}}-\frac{2\sqrt{2\pi}\mu}{\sqrt{c}}T\rpar{\frac{\mu}{\sigma},\sqrt{\frac{c}{\sigma^2-c}}}.\qedhere
\end{align*}
\end{proof}

\begin{lemma}\label{lemma:special_integral_alt}
For all $\mu\in\R$, $\sigma\in\R_{>0}$ and $c\in(0,\sigma^2)$, 
\begin{align*}
&\int_0^{1}\frac{\sqrt{\sigma^2-\tau c}}{\sqrt{\tau}}\phi\rpar{\frac{\mu}{\sqrt{\sigma^2-\tau c}}}\diff\tau\\
=~&\sqrt{\sigma^2-c}\cdot\phi\rpar{\frac{\mu}{\sqrt{\sigma^2-c}}}+\frac{\sqrt{2\pi}\mu\sigma}{\sqrt{c}}\phi\rpar{\frac{\mu}{\sigma}} \spar{\Phi\rpar{\frac{\mu}{\sigma}\sqrt{\frac{c}{\sigma^2-c}}}-\half}+\frac{\sqrt{2\pi}}{\sqrt{c}}\rpar{\sigma^2-\mu^2}T\rpar{\frac{\mu}{\sigma},\sqrt{\frac{c}{\sigma^2-c}}}.
\end{align*}
\end{lemma}

\begin{proof}[Proof of Lemma~\ref{lemma:special_integral_alt}]
Similar to the first step of Lemma~\ref{lemma:special_integral}, with the change of variable $\tau\leftarrow \frac{\tau^2\sigma^2}{(1+\tau^2)c}$, the LHS of the objective equality can be converted as
\begin{equation*}
\int_0^{1}\frac{\sqrt{\sigma^2-\tau c}}{\sqrt{\tau}}\phi\rpar{\frac{\mu}{\sqrt{\sigma^2-\tau c}}}\diff\tau=\frac{2\sigma^2}{\sqrt{2\pi c}}\int_0^{\sqrt{\frac{c}{\sigma^2-c}}}\frac{1}{(1+\tau^2)^2}\exp\rpar{-\frac{\mu^2(1+\tau^2)}{2\sigma^2}}\diff\tau.
\end{equation*}
For brevity, denote the obtained RHS as $\Diamond$. The objective now is to further simplify this term, specifically relating it to Owen's T function, $T(x,y)\defeq\frac{1}{2\pi}\int_0^y\frac{e^{-\half x^2(1+z^2)}}{1+z^2}\diff z$. 

For notational convenience, define $v=\frac{\mu^2}{2\sigma^2}\geq 0$. Taking the following derivative yields
\begin{align*}
\frac{\diff}{\diff \tau}\spar{\frac{\tau}{1+\tau^2}\exp\rpar{-v(1+\tau^2)}}&=\spar{\frac{1-2v\tau^2}{1+\tau^2}-\frac{2\tau^2}{(1+\tau^2)^2}}\exp\rpar{-v(1+\tau^2)}\\
&=\spar{-\frac{1+2v\tau^2}{1+\tau^2}+\frac{2}{(1+\tau^2)^2}}\exp\rpar{-v(1+\tau^2)}\\
&=\spar{-2v+\frac{2v-1}{1+\tau^2}+\frac{2}{(1+\tau^2)^2}}\exp\rpar{-v(1+\tau^2)},
\end{align*}
therefore
\begin{equation*}
\frac{\exp\rpar{-v(1+\tau^2)}}{(1+\tau^2)^2}=\half \frac{\diff}{\diff \tau}\spar{\frac{\tau}{1+\tau^2}\exp\rpar{-v(1+\tau^2)}}+v\exp\rpar{-v(1+\tau^2)}+\rpar{\half-v}\frac{\exp\rpar{-v(1+\tau^2)}}{1+\tau^2}.
\end{equation*}
Integrating it with respect to $\tau$ gives
\begin{align*}
&\int_0^y\frac{\exp\rpar{-v(1+\tau^2)}}{(1+\tau^2)^2}\diff \tau\\
=~&\half \frac{\tau}{1+\tau^2}\exp\rpar{-v(1+\tau^2)}\bigg|_{\tau=0}^{y}+ve^{-v}\int_0^y\exp\rpar{-v\tau^2}\diff \tau+\rpar{\half-v}\int_0^y\frac{\exp\rpar{-v(1+\tau^2)}}{1+\tau^2}\diff\tau\\
=~&\half \frac{y}{1+y^2}\exp\rpar{-v(1+y^2)}+\sqrt{\pi v}e^{-v} \spar{\Phi(\sqrt{2v}y)-\half}+2\pi\rpar{\half-v}T(\sqrt{2v},y).
\end{align*}
Specializing it to $\Diamond$, 
\begin{align*}
&\Diamond\\
=~&\frac{2\sigma^2}{\sqrt{2\pi c}}\spar{\half \frac{y}{1+y^2}\exp\rpar{-v(1+y^2)}+\sqrt{\pi v}e^{-v} \spar{\Phi(\sqrt{2v}y)-\half}+2\pi\rpar{\half-v}T(\sqrt{2v},y)}\bigg|_{y=\sqrt{\frac{c}{\sigma^2-c}},v=\frac{\mu^2}{2\sigma^2}}\\
=~&\frac{\sqrt{\sigma^2-c}}{\sqrt{2\pi }}\exp\rpar{-\frac{\mu^2}{2(\sigma^2-c)}}+\frac{\mu\sigma}{\sqrt{c}}\exp\rpar{-\frac{\mu^2}{2\sigma^2}} \spar{\Phi\rpar{\frac{\mu}{\sigma}\sqrt{\frac{c}{\sigma^2-c}}}-\half}\\
&\quad+\frac{\sqrt{2\pi}}{\sqrt{c}}\rpar{\sigma^2-\mu^2}T\rpar{\frac{\mu}{\sigma},\sqrt{\frac{c}{\sigma^2-c}}}\\
=~& \sqrt{\sigma^2-c}\cdot\phi\rpar{\frac{\mu}{\sqrt{\sigma^2-c}}}+\frac{\sqrt{2\pi}\mu\sigma}{\sqrt{c}}\phi\rpar{\frac{\mu}{\sigma}} \spar{\Phi\rpar{\frac{\mu}{\sigma}\sqrt{\frac{c}{\sigma^2-c}}}-\half}+\frac{\sqrt{2\pi}}{\sqrt{c}}\rpar{\sigma^2-\mu^2}T\rpar{\frac{\mu}{\sigma},\sqrt{\frac{c}{\sigma^2-c}}}.\qedhere
\end{align*}
\end{proof}



\end{document}